\documentclass[10pt]{article} % For LaTeX2e

\usepackage[accepted]{rlj} % Should be uncommented for the camera-ready
\usepackage{amssymb}            % Defines common symbols like \mathbb R
\usepackage{mathtools}          % Extends amsmath, providing common math tools
\usepackage{mathrsfs}           % Enables \mathscr, which can work in cases that \mathcal does not
\usepackage{graphicx}           % For including images
\usepackage{subcaption}         % Allows for the use of subfigures and subcaptions
\usepackage{url}                % For displaying URLs
\usepackage{lipsum}             % For placeholder text

\usepackage{soul}
\usepackage{sgame}
\usepackage{multirow}
\usepackage{cancel}
\usepackage{wrapfig}
\usepackage{booktabs}
\usepackage{pifont}
\usepackage{amsthm}
\usepackage{thmtools}
\usepackage{thm-restate}
\usepackage{etoolbox}
\BeforeBeginEnvironment{proof}{\vspace{-0.8\baselineskip}}
\AfterEndEnvironment{proof}{\vspace{-0.8\baselineskip}}

\newcommand{\E}{\mathbb{E}}
\newcommand{\va}{\vec{a}}
\newcommand{\vphi}{\varphi}
\newcommand{\bel}[1]{b^{(#1)}}
\newcommand{\sg}{\operatorname{sg}}
\newcommand{\dgp}{d^{\vec{\pi}}_{\gamma'}}

\theoremstyle{plain}
\newtheorem{theorem}{Theorem}[section]

\theoremstyle{definition}
\newtheorem{definition}[theorem]{Definition}

\theoremstyle{remark}

\usepackage{algcompatible}
\usepackage{algorithm}

\usepackage[textsize=tiny]{todonotes}

\title{ACPO: Agent-Chained Policy Optimization for 
Multi-Agent Reinforcement Learning}

\setrunningtitle{ACPO: Agent-Chained Policy Optimization for 
Multi-Agent Reinforcement Learning}

\author{Daiki E. Matsunaga\textsuperscript{1}, Junho Na\textsuperscript{1}, Tri Wahyu Guntara\textsuperscript{2}, Scott Sanner\textsuperscript{3 6}, 
\\Pascal Poupart\textsuperscript{4 6$\dagger$}, Jongmin Lee\textsuperscript{5 7$\dagger$}, Kee-Eung Kim\textsuperscript{1 7$\dagger$}}

\emails{\{d.e.matsunaga, jhna9803, kekim\}@kaist.ac.kr, wahyu.guntara@krafton.com, ssanner@mie.utoronto.ca, ppoupart@uwaterloo.ca, jongminlee@yonsei.ac.kr}

\affiliations{
$^{1}$\textbf{KAIST}
$^{2}$\textbf{KRAFTON}
$^{3}$\textbf{University of Toronto}
$^{4}$\textbf{University of Waterloo}
$^{5}$\textbf{Yonsei University}
$^{6}$\textbf{Vector Institute}
$^{7}$\textbf{National AI Research Lab}
}

\contribution{
    We introduce the Agent-Chained Belief MDP (AC-BMDP), a formalization that combines serialization with belief augmentation to enable optimal joint policy learning under Centralized Training with Decentralized Execution (CTDE).
    }
    {
Prior work on serialization~\citep{kovarik2023factoredposg, Peralez2025simultaneousmove} showed that simultaneous MMDPs can be converted to sequential problems, but did not address the partial observability arising from unobserved preceding agent actions  under CTDE. Our formalism is also motivated by a classical result in game theory which states that any simultaneous game can be equivalently recast as a serialized decision process with imperfect information~\citep[5.2.2]{shoham2008multiagent}.  
}

\contribution{
    We develop Agent-Chained Policy Optimization (ACPO), a practical actor-critic framework where independent policy updates together constitute a single step on the joint policy gradient. In tabular settings, this converges to the optimal joint policy in the underlying Multi-Agent MDP~(MMDP) rather than to Nash Equilibria.
    }{
Alternating policy optimization methods~\citep{kuba2022hatrpo, zhong2024haml} guarantee convergence to Nash Equilibria, which can be arbitrarily suboptimal in cooperative settings. Independent policy optimization methods~\citep{lowe2017maddpg, yu2022mappo} lack convergence guarantees without strong structural assumptions such as value decomposition.
}
\contribution{
    We introduce on-policy (PPO-based) and off-policy (TD3-based) instantiations of ACPO and demonstrate empirically that they consistently outperforms state-of-the-art baselines across three standard benchmarks (Multi-Robot Warehouse, SMACv2, MA-MuJoCo), with performance gains that scale with the number of agents.
    }{
    None}

 % Your keywords

\summary{
Cooperative tasks in Multi-Agent Reinforcement Learning~(MARL) require agents to collectively maximize a shared return. Under the Centralized Training with Decentralized Execution~(CTDE) paradigm, policy gradients have remained difficult to compute directly. Prior methods largely follow two approaches: independent factorized updates with centralized critics, which lack general joint-improvement guarantees without value decomposition assumptions, or alternating best-response updates, which can converge to suboptimal Nash Equilibria.

In this paper, we show the joint policy gradient admits an exact decentralized decomposition of per-agent terms, each formed from per-agent score functions and decentralized critics. Based on this decomposition, we develop Agent-Chained Policy Optimization~(ACPO), where actors are trained independently, with their updates together constituting a single step on the joint policy gradient.
Central to this result is a serialized view of the simultaneous joint decision in which agents commit actions one at a time, each conditioning on a belief over preceding actions that ties the independent per-agent updates into a single joint step.
We evaluate on-policy and off-policy instantiations of ACPO on Multi-Robot Warehouse, SMACv2, and MA-MuJoCo, where it outperforms strong baselines, with the gap widening as the number of agents grows.
}

\begin{document}

\makeCover  % Create the cover page
\maketitle  % Make the title section
\begingroup
\renewcommand{\thefootnote}{}%
\footnotetext{$\dagger$ Equal advising.}%
\endgroup
\begin{abstract}
Cooperative tasks in Multi-Agent Reinforcement Learning~(MARL) require agents to collectively maximize a shared return. Under the Centralized Training with Decentralized Execution~(CTDE) paradigm, policy gradients have remained difficult to compute directly. Prior methods largely follow two approaches: independent factorized updates with centralized critics, which lack general joint-improvement guarantees without value decomposition assumptions, or alternating best-response updates, which can converge to suboptimal Nash Equilibria.
In this paper, we show the joint policy gradient admits an exact decentralized decomposition of per-agent terms, each formed from per-agent score functions and decentralized critics. Based on this decomposition, we develop Agent-Chained Policy Optimization~(ACPO), where actors are trained independently, with their updates together constituting a single step on the joint policy gradient.
Central to this result is a serialized view of the simultaneous joint decision in which agents commit actions one at a time, each conditioning on a belief over preceding actions that ties the independent per-agent updates into a single joint step.
We evaluate on-policy and off-policy instantiations of ACPO on Multi-Robot Warehouse, SMACv2, and MA-MuJoCo, where it outperforms strong baselines, with the gap widening as the number of agents grows.
\end{abstract}

%%%%%%%%%%%%%%%%%%%%%%%%%%%%%%%%%%%%%%%%%%%%%%%%%%%%%%%%%%%%%%%%
%% Introduction
%%%%%%%%%%%%%%%%%%%%%%%%%%%%%%%%%%%%%%%%%%%%%%%%%%%%%%%%%%%%%%%%
\section{Introduction}
Cooperative multi-agent reinforcement learning (MARL) seeks decentralized
policies that jointly maximize a shared return, in domains ranging from
autonomous vehicle fleets~\citep{zhang2024multiagentreinforcementlearningautonomous}
and traffic signal control~\citep{chu2020traffic} to fleet
management~\citep{lin2018fleet}, power
networks~\citep{wang2021powernetworks}, and Large Language
Models~\citep{wu2024autogen, liu2025llmcollaborationmultiagentreinforcement}.
Underlying every cooperative MARL problem is a single-agent Markov
Decision Process over the joint action space~(Multi-Agent MDP)~\citep{boutilier1996mmdp}, whose optimal policy maximizes the shared
return. Single-agent policy gradient
methods~\citep{ddpg2016, schulman15trpo,schulman2017proximalpolicyoptimizationalgorithms} extend to this MDP in
principle,
with the joint policy playing the role of a single centralized
agent's policy. In real-world deployments, however, this joint policy
must factorize across agents and each agent must act on its own information
without inter-agent communication. This is the setting of Centralized
Training with Decentralized Execution
(CTDE)~\citep{lowe2017maddpg,Foerster2018coma}.

Computing the Multi-Agent MDP~(MMDP) policy gradient under decentralized execution has
proven to be a challenge. Existing CTDE methods either do not handle the non-stationarity incurred by independent policy updates or solve a different problem.
\emph{Independent policy optimization}
(MAPPO~\citep{yu2022mappo}, MADDPG~\citep{lowe2017maddpg}) trains each
decentralized actor against a centralized critic conditioned on the
joint action. As actors are trained independently without accounting for the non-stationarity incurred by the other agents' changing policies, joint improvement guarantees do not hold in general~(Appendix~\ref{appendix:limitations_independent_policy_updates}), without making certain assumptions on value decomposition.
\emph{Alternating policy optimization}
(HATRPO/HAPPO~\citep{kuba2022hatrpo, zhong2024haml}) sidesteps the
MMDP by updating each agent against a per-agent reduced MDP
with the others' policies fixed. This monotonically improves the joint
return at each step, but the underlying problem is now a sequence of
unilateral best-response problems, and the fixed point is a Nash
Equilibrium that can be arbitrarily far from the optimal joint
policy\footnote{In game theory terminology, the \emph{optimal joint
policy} in cooperative MARL is the social optimum that maximizes
welfare.}~(Table~\ref{matrix_game_policy_values_main_text}).

In this paper, we show that the policy gradient in the MMDP can be computed directly under the decentralized execution constraint, via an exact decomposition into decentralized per-agent terms. The decomposition is enabled by augmenting each agent's state with a belief over the actions committed by preceding agents in a serialized view of the joint decision. We prove the \emph{Multi-Agent Policy Gradient Decomposition Theorem}~(Theorem~\ref{thm-mapg-decomposition}): the joint policy gradient decomposes into per-agent terms, each involving only a per-agent score function and a per-agent critic, with the  updates coupled through beliefs. Unlike value decomposition methods~\citep{sunehag2018vdn, rashid2018qmix,
wang2021dop, zhang2021fop}, this decomposition is made without any structural assumption on the joint value function. The per-agent critics are determined by a chained Bellman recursion in which each agent bootstraps from the next
agent in the chain (Figure~\ref{figure:value_network_architecture}).

We instantiate the decomposition as \emph{Agent-Chained Policy Optimization} (ACPO), a general actor-critic framework with on-policy PPO-based~\citep{schulman2017proximalpolicyoptimizationalgorithms} and off-policy TD3-based~\citep{DBLP:conf/icml/FujimotoHM18} variants. Each agent updates its actor against its own decentralized critic, and the independent per-agent updates collectively conduct a single gradient step on the MMDP joint policy.
On Multi-Robot Warehouse, SMACv2, and MA-MuJoCo, ACPO outperforms strong baselines, with the gap widening substantially as the number of agents grows.

%%%%%%%%%%%%%%%%%%%%%%%%%%%%%%%%%%%%%%%%%%%%%%%%%%%%%%%%%%%%%%%%
%% Background
%%%%%%%%%%%%%%%%%%%%%%%%%%%%%%%%%%%%%%%%%%%%%%%%%%%%%%%%%%%%%%%%
\section{Background}
\subsection{Multi-Agent MDP}

We consider a cooperative multi-agent reinforcement learning (MARL) setting
with $\mathcal{{N}}=\left\{ 1,\ldots,N\right\} $ agents, formally modeled as a Multi-Agent Markov Decision Process~(MMDP)~\citep{boutilier1996mmdp}. At time
step $t$, each agent $i\in\mathcal{N}$ simultaneously takes action $a_{t}^{(i)}\in \mathcal A^{(i)}$ 
sampled from individual policy $\pi^{(i)}(a_{t}^{(i)}\mid s_t)$ where $s_t \in \mathcal S$ is the state.
The state transition is Markovian,
i.e. the next state $s_{t+1}$ is given by transition function $T(s_{t+1}|s_{t},\vec{a}_{t})$
where $\vec{a}_{t}$ is the joint action $\vec{a}_{t}=[a_{t}^{(1)},\ldots,a_{t}^{(N)}]$.
Each agent receives shared reward $r_{t}$ generated by the common
reward function $R(s_{t},\vec{a}_{t})$. 

The goal of cooperative MARL is to find a set of agent policies
$\vec{\pi}=[\pi^{(1)},\ldots,\pi^{(N)}]$ that maximize the total
expected return $J(\vec \pi)=\mathbb{E}_{\vec\pi}\left[\sum_{t}\gamma^{t}r_{t}\right]$
where $\gamma$ is the discount factor. 
Following standard practice in cooperative MARL, we adopt the Centralized Training with Decentralized Execution (CTDE) setting where policies are trained with shared information, but executed independently\footnote{See Appendix~\ref{appendix:ctde} for further discussion of CTDE.}.

\subsection{Previous CTDE Approaches in Cooperative MARL}
\label{subsection:previous_approaches}

\paragraph{{Independent Policy Optimization}}

One way to solve MMDPs is to apply single-agent RL on the joint action space with independently factorized policies~\citep{lowe2017maddpg, yu2022mappo}. Each policy is trained via an objective of the form $J(\pi^{(i)})=\mathbb{E}_{\pi^{(i)}}\mathbb{E}_{\vec\pi^{-i}}[Q(s,a^{(i)},\vec a^{-i})]$, where $Q$ is the centralized action-value function, and $\vec a^{-i}, \vec\pi^{-i}$ are the joint actions and joint policies excluding agent $i$. The $N$ agents each maximize $J(\pi^{(i)})$ independently. MADDPG~\citep{lowe2017maddpg} instantiates this with deterministic policy gradients via a learned $Q$~\citep{ddpg2016}, and MAPPO~\citep{yu2022mappo} uses a centralized state-value function with GAE~\citep{schulman2016gae,schulman2017proximalpolicyoptimizationalgorithms} to compute advantages $A(s,a^{(i)},\vec a^{-i})$. Both methods can train policies in parallel and scale well in practice. However, there is no general guarantee of joint policy improvement even in tabular settings. As we show in Appendix~\ref{appendix:limitations_independent_policy_updates}, simple counter-examples exist where these methods diverge. Convergence holds only in restricted settings where the optimal action-value decomposes such that independent updates align with joint improvement.

\paragraph{Alternating Policy Optimization}

A second approach defines a reduced MDP for each agent and learns the best response policy, yielding monotonic improvement of the joint policy and convergence to a Nash Equilibrium (NE)~\citep{bertsekas2020multiagentvalueiterationalgorithms}. Each agent $i$ alternates and solves $\langle\mathcal{S},\mathcal{A}^{(i)},T^{(i)},R^{(i)}\rangle$, where $T^{(i)}(s_{t+1}\mid s_t,a^{(i)}_t):=\mathbb{E}_{\vec\pi^{-i}}[T(s_{t+1}\mid s_t,a^{(i)}_t,\vec a^{-i})]$ and $R^{(i)}(s_t,a^{(i)}_t):=\mathbb{E}_{\vec\pi^{-i}}[R(s_t,a^{(i)}_t,\vec a^{-i})]$ marginalize over the other agents' actions under their most recent policies $\vec\pi^{-i}$. When $\pi^{(i)}$ is being updated, the previous agent policies $\vec\pi^{<i}$ have already been updated. HATRPO~\citep{kuba2022hatrpo} and HAPPO~\citep{zhong2024haml} instantiate alternating best-response updates with trust-region~\citep{schulman15trpo} or clipped-surrogate~\citep{schulman2017proximalpolicyoptimizationalgorithms} optimization. The advantage $A^{(i)}(s,a^{(i)})$ is well-defined in the reduced MDP, and bounded sequential updates yield monotonic improvement.
The fixed point, however, is only an NE, which can be arbitrarily suboptimal in cooperative settings. Without a mechanism for joint updates, agents have no way to escape suboptimal NEs. Sequential updates also scale poorly since training time increases with the number of agents.

\subsection{Limitations of Existing CTDE Approaches}
\label{subsection:limitation_of_previous_approaches}
\begin{table*}[t!]
\scriptsize{
% \captionsetup{font=scriptsize}
\centering
\begin{tabular}{c|c|c|c|}
\multicolumn{1}{c}{ } & \multicolumn{1}{c}{ $A$} & \multicolumn{1}{c}{$B$} & \multicolumn{1}{c}{ $C$} \\
 \cline{2-4}
$A$ & $5$ &  $-20$ & $-20$\\
 \cline{2-4}
$B$ & $-20$ &  $10$ & $-20$ \\
\cline{2-4}
$C$ & $-20$ &  $-20$ & $20$\\
\cline{2-4}
\multicolumn{1}{c}{ } &\multicolumn{3}{c}{3x3 Matrix Game} \\
\end{tabular}
% \quad
\\
%------------------------------------------------
\begin{tabular}{c|c|c|c|}
 \multicolumn{1}{c}{ } &  \multicolumn{1}{c}{ $A$} & \multicolumn{1}{c}{ $B$} & \multicolumn{1}{c}{ $C$} \\
 \cline{2-4}
$A$ &$1$ & $0$ & $0$ \\
\cline{2-4}
 $B$ &$0$ & $0$ & $0$ \\
\cline{2-4}
 $C$ & $0$ & $0$ & $0$ \\
\cline{2-4}
 \multicolumn{1}{c}{ } &\multicolumn{3}{c}{Independent PO} \\
\end{tabular}
\quad
%------------------------------------------------
\begin{tabular}{|c|c|c|}
 \multicolumn{1}{c}{ $A$} & \multicolumn{1}{c}{ $B$} & \multicolumn{1}{c}{ $C$} \\
 \cline{1-3}
$1$ & $0$ & $0$ \\
\cline{1-3}
 $0$ & $0$ & $0$ \\
\cline{1-3}
 $0$ & $0$ & $0$ \\
\cline{1-3}
\multicolumn{3}{c}{Alternating PO} \\
\end{tabular}
\quad
%------------------------------------------------
\begin{tabular}{|c|c|c|}
 \multicolumn{1}{c}{ $A$} & \multicolumn{1}{c}{ $B$} & \multicolumn{1}{c}{ $C$} \\
 \cline{1-3}
$1$ & $0$ & $0$ \\
\cline{1-3}
 $0$ & $0$ & $0$ \\
\cline{1-3}
 $0$ & $0$ & $0$ \\
\cline{1-3}
\multicolumn{3}{c}{HASPI ($\alpha$=1)} \\
\end{tabular}
\quad
%------------------------------------------------
\begin{tabular}{|c|c|c|}
 \multicolumn{1}{c}{ $A$} & \multicolumn{1}{c}{ $B$} & \multicolumn{1}{c}{ $C$} \\
 \cline{1-3}
$0$ & $0$ & $0$ \\
\cline{1-3}
 $0$ & $0$ & $0$ \\
\cline{1-3}
 $0$ & $0$ & $1$ \\
\cline{1-3}
\multicolumn{3}{c}{HASPI ($\alpha$=5)} \\
\end{tabular}
\quad
%------------------------------------------------
\begin{tabular}{|c|c|c|}
 \multicolumn{1}{c}{ $A$} & \multicolumn{1}{c}{ $B$} & \multicolumn{1}{c}{ $C$} \\
 \cline{1-3}
$0$ & $0$ & $0$ \\
\cline{1-3}
 $0$ & $0$ & $0$ \\
\cline{1-3}
 $0$ & $0$ & $1$ \\
\cline{1-3}
\multicolumn{3}{c}{ACPO (Ours) } \\
\end{tabular}
% \quad
\caption{A 3x3 Matrix Game and converged policy values when initialized to $\pi^{(1)}(A)=\pi^{(2)}(A)=0.6$. ACPO~(Ours) converges to the  optimal joint policy which selects $(C, C)$ regardless of the initialization. }
\label{matrix_game_policy_values_main_text}
}
% \vspace{-10pt}
\end{table*}

% \vspace{-0.33cm}
Consider the simple Matrix Game in Table~\ref{matrix_game_policy_values_main_text} with 2 agents and 3 actions, which was considered in \citet{liu2024hasac}.  
The three NEs are $(A, A), (B, B), (C, C)$, since there is no incentive for either agent to change its action at each of those NEs. $(C,C)$ is a social optimum of the game since it achieves the highest return. As shown in Table~\ref{matrix_game_policy_values_main_text},  independent policy optimization~(Independent PO) and alternating policy optimization~(Alternating PO), which are the foundations for MAPPO and HAPPO, respectively, cannot escape suboptimal NEs.
The only way they can converge to $(C,C)$ is for the policy to be initialized with high probability towards $(C,C)$. 
Similarly, Heterogeneous-Agent Soft Policy Iteration~(HASPI)~\citep{liu2024hasac} can only find the optimal policy for specific combinations of the entropy parameter $\alpha$ and the initialization of the policy (e.g., $\alpha=5$ and $\pi^{(1)}(A)=\pi^{(2)}(A)=0.6$), if it happens to coincide with the QRE. As with NEs, there is no guarantee that the QRE will coincide with the optimal policy.

Our goal is to derive a principled algorithm which directly targets the optimal joint policy in the MMDP. 
In Section~\ref{section:acpo}, we present Agent-Chained Policy Optimization~(ACPO) which converges to the optimal joint policy under tabular domains. Detailed analysis on how tabular ACPO solves the Matrix Game (regardless of the policy initialization)
is provided in Appendix~\ref{appendix:exact_calculation_matrix}.

\section{Agent-Chained Belief MDP~(AC-BMDP)}
\label{section:ac_mdp}
A classical result in game theory states that any simultaneous game can be equivalently recast as a serialized decision process with imperfect information~\citep[5.2.2]{shoham2008multiagent}. This equivalence holds under any agent order. In this section, we instantiate this in the context of cooperative MARL and formalize it as the AC-BMDP. 

We start by decomposing each environment step $t$ into $N$ \emph{micro-steps} indexed by the acting agent $i \in \{1,\dots,N\}$. Let $\vec a^{<i}_t := [a^{(1)}_t,\dots,a^{(i-1)}_t]$ denote the actions committed by agents $1$ through $i-1$ within the current environment step, with $\vec a^{<1}_t=\emptyset$. The serialized fully-observed state is $[s_t,\vec a^{<i}_t]$. The micro-step transitions are deterministic and rewards are $0$ until the final agent commits~(Details in Appendix~\ref{appendix:subsection_serialization_proof}).

\smallskip
\noindent\textbf{State Augmentation with Belief over Preceding Actions.}
We now make explicit the structure of standard policy parameterizations in RL. Each agent's policy decomposes into a deterministic mapping from state to action distribution, followed by stochastic sampling,
\[
\pi^{(i)}:\mathcal S\to \Delta(\mathcal A^{(i)}), \qquad
\varphi^{(i)}_t=\pi^{(i)}(s_t), \qquad
a^{(i)}_t \sim \varphi^{(i)}_t.
\]
The \emph{action distribution} $\varphi^{(i)}_t$ can be the probability vector of a softmax policy for discrete action spaces, or the mean and standard deviation of a Gaussian for continuous action spaces.\footnote{This parameterization encompasses both algorithms for learning stochastic policies~\citep{schulman2017proximalpolicyoptimizationalgorithms} and deterministic target policies with a stochastic behavior policy~\citep{DBLP:conf/icml/FujimotoHM18}.}

While the realized action $a^{(i)}_t$ is private, the action distribution $\varphi^{(i)}_t$ is deterministically computed from the shared state and is therefore accessible to other agents that have knowledge of $\pi^{(i)}$, with no inter-agent communication required. This makes $\varphi^{(i)}_t$ fully compatible with decentralized execution.\footnote{Knowledge of other agents' policies is a standard assumption in game theory. In practice, $\vec\varphi^{<i}_t$ can be computed exactly by querying a shared policy with each preceding agent's index, or by maintaining copies of the preceding agents' policies. In partially observable settings it must be approximated (Section~\ref{section:practical_implementation}).}

Let $\vec\varphi^{<i}_t := [\varphi^{(1)}_t,\dots,\varphi^{(i-1)}_t]$. Since actions are sampled independently across agents, the belief over preceding realized actions factorizes as
\[
b^{(i)}_t(\vec a^{<i}_t)
\;\triangleq\; \Pr(\vec a^{<i}_t \mid s_t,\vec\varphi^{<i}_t)
\;=\; \prod_{j=1}^{i-1}\varphi^{(j)}_t(a^{(j)}_t),
\]
with $b^{(1)}_t = \emptyset$ for agent 1. The belief is fully determined by the compact sufficient statistic $\vec\varphi^{<i}_t$.

\smallskip
\begin{restatable}{definition}{AC-BMDP}
\label{def:ac-bmdp}(AC-BMDP)
We define an MDP with augmented state and action spaces 

$\langle \mathcal S \times \mathcal B^{(i)}, \Phi^{(i)}, R, T, \gamma' \rangle$, where $\mathcal B^{(i)}$ is the space of beliefs over $\vec a^{<i}$ (equivalently represented by $\vec\varphi^{<i}$). The state at micro-step $i$ is $[s_t,b^{(i)}_t]$, with $b^{(1)}_t$ empty. The action space is the set of action distributions $\Phi^{(i)}:=\Delta(\mathcal A^{(i)})$. Agent $i$'s policy selects $\varphi^{(i)} \in \Phi^{(i)}$, and the subsequent sample $a^{(i)} \sim \varphi^{(i)}$ is treated as part of the environment. This makes the action publicly observable, allowing subsequent agents to update their beliefs.
\end{restatable}

\emph{Reward.} Rewards follow serialization but are marginalized over the unobserved realized actions,

\[
R\!\left([s_t,b^{(N)}_t], \varphi^{(N)}_t\right)=
\sum_{\vec a^{<N}_t} b^{(N)}_t(\vec a^{<N}_t)\sum_{a^{(N)}_t}\varphi^{(N)}_t(a^{(N)}_t)\, R(s_t,\vec a_t)
\]
Rewards for agents $i\in\{1,\dots,N-1\}$ are always 0.

\emph{Transition.} For $i\in\{1,\dots,N-1\}$, the environment state is unchanged and only the belief is updated deterministically,

\[
T\!\left([s_t,b^{(i+1)}_t]\mid [s_t,b^{(i)}_t],\varphi^{(i)}_t\right)
=
\begin{cases}
1, & b^{(i+1)}_t=\tau([s_t,b^{(i)}_t],\varphi^{(i)}_t),\\
0, & \text{otherwise},
\end{cases}
\]
where $\tau$ is the belief update~(Appendix~\ref{appendix:belief_update_rule}).
At $i=N$, the environment transitions according to the original MMDP dynamics marginalized over the joint action,
\[
T\!\left([s_{t+1},\emptyset]\mid [s_t,b^{(N)}_t],\varphi^{(N)}_t\right)
=
\sum_{\vec a^{<N}_t} b^{(N)}_t(\vec a^{<N}_t)\sum_{a^{(N)}_t}\varphi^{(N)}_t(a^{(N)}_t)\, T(s_{t+1}\mid s_t,\vec a_t).
\]

\emph{Discount factor.} The discount factor is $\gamma' = \gamma^{1/N}$ for each micro-step.

\emph{Belief update.} The serialized micro-step transition deterministically appends the realized action, so $\tau$  simplifies to the product-form update

\[
b^{(i+1)}_t(\vec a^{<i}_t,a^{(i)}_t)= b^{(i)}_t(\vec a^{<i}_t)\,\varphi^{(i)}_t(a^{(i)}_t).
\]
Equivalently, the belief-state can be represented by the preceding action distributions via the deterministic append $\vec\varphi^{<i+1}_t=[\vec\varphi^{<i}_t,\varphi^{(i)}_t]$, with $b^{(i)}_t$ recovered through the factorization above. 

The Bellman operator under the AC-BMDP is given by the following.

\begin{restatable}{definition}{DefbellmanOperator}
\label{def:bellman_operators}(Agent-Chained Bellman Operators)
    \begin{flalign*}
    &\forall i \in \{1, \dots, N-1 \} \\
    &(\mathcal{T}^{\vec{\pi}}Q^{(i+1)})([s, b^{(i)}], \varphi^{(i)})
        :=\gamma' \mathbb{E}_{\substack{b^{(i+1)} = T([s, b^{(i)}], \varphi^{(i)})\\\varphi^{(i+1)} \sim \pi^{(i+1)}(\cdot \mid s, b^{(i+1)})
        }} \left[Q^{(i+1)}\left([s, b^{(i+1)}], \varphi^{(i+1)}\right) \right]  
        % \text{if } i \in \{1, \dots, N-1 \},
        \\
    &(\mathcal{T}^{\vec{\pi}}Q^{(1)})([s, b^{(N)}], \varphi^{(N)}) :=R\left([s, b^{(N)}], \varphi^{(N)}\right)
          + \gamma' \mathbb{E}_{\substack{s' \sim T(\cdot \mid [s, b^{(N)}], \varphi^{(N)})\\\varphi^{(1)} \sim \pi^{(1)}(\cdot \mid s')
        }}\left[Q^{(1)}\left(s', \varphi^{(1)}\right) \right]
    \end{flalign*}
\end{restatable}
The key structure is \emph{agent-chaining}. The target for $Q^{(i)}$ is $Q^{(i+1)}$, giving each agent its own decentralized value function without requiring value decomposition assumptions~\citep{sunehag2018vdn, rashid2018qmix, zhang2021fop} or convexity assumptions~\citep{wang2022shaq}. Each $Q^{(i)}([s, b^{(i)}], \varphi^{(i)})$ captures how agent $i$'s action affects the next agent's value, providing a natural mechanism for credit assignment.

\smallskip
\noindent\textbf{Theoretical Properties of the AC-BMDP.}
The AC-BMDP is an equivalent serialized representation of the MMDP with the state and action spaces augmented. The optimal policy of the AC-BMDP is also optimal in the MMDP. 
It also follows that in tabular settings, running policy iteration (\textit{Agent-Chained Policy Iteration}) on the AC-BMDP converges to the optimal joint policy in the underlying Multi-Agent MDP~(MMDP) rather than to Nash Equilibria. We formalize these  results in  Appendix~\ref{appendix:policy_iteration_proofs}.
\section{Agent-Chained Policy Optimization~(ACPO)}
\label{section:acpo}
Section~\ref{section:ac_mdp} recasts the MMDP as the AC-BMDP. We now show that this construction yields an exact decomposition of the joint policy gradient into $N$ independent per-agent terms, each computable from a per-agent score function and a decentralized critic. The decomposition is the foundation of ACPO, where every agent improves its own policy against its own critic independently, and the updates collectively constitute a gradient step on the joint MMDP objective.

\paragraph{Agent-Chained Policy Gradient}
Applying the single-agent policy gradient theorem to the AC-BMDP yields
\begin{equation}
\nabla_\theta J^{\text{AC}}(\vec{\pi}_\theta)
= \mathbb{E}_{[s,b^{(i)}] \sim d_{\gamma'}^{\vec \pi},\, \varphi^{(i)} \sim \pi^{(i)}_\theta}
\!\left[\nabla_\theta \log \pi^{(i)}_\theta(\varphi^{(i)} \mid s, b^{(i)}) \cdot Q^{(i)}([s, b^{(i)}], \varphi^{(i)})\right],
\label{eq:lifted_pg}
\end{equation}
where $d_{\gamma'}^{\vec \pi}$ denotes the occupancy measure in the AC-BMDP under a policy $\vec \pi$. 

The MMDP joint policy gradient $\nabla_\theta J(\vec\pi_\theta)$ can be written as a sum of $N$ such per-agent terms.

\begin{restatable}{theorem}{ThmMapgDecomposition}
\label{thm-mapg-decomposition}(Multi-Agent Policy Gradient Decomposition Theorem)
\begin{equation}
\nabla_\theta J(\vec{\pi}_\theta)
\;=\; \frac{1}{(\gamma')^{N-1}}
\sum_{i=1}^N \mathbb{E}_{\substack{[s, b^{(i)}] \sim d_{\gamma'}^{\vec \pi} \\ \varphi^{(i)} \sim \pi^{(i)}_\theta(\cdot \mid s, b^{(i)})}}
\!\left[
\nabla_\theta \log \pi^{(i)}_\theta\!\left(\varphi^{(i)} \,\big|\, s, b^{(i)}\right)
\cdot
Q^{(i)}\!\left([s, b^{(i)}], \varphi^{(i)}\right)
\right].
\label{eq:pg-decomposition}
\end{equation}
\end{restatable}
\begin{proof}
    See Appendix~\ref{appendix:decomposition-proof}.
\end{proof}

The belief $b^{(i)}$ makes each term in \eqref{eq:pg-decomposition} a function of agent $i$ alone: $Q^{(i)}([s, b^{(i)}], \varphi^{(i)})$ depends on preceding agents only through their action distributions encoded in $b^{(i)}$, and never on their realized actions. Each per-agent gradient is computable from agent $i$'s own decentralized critic. The $N$ updates therefore can be conducted simultaneously and the collective update is a single gradient step on the joint MMDP objective.

The PPO instantiation we present next is a practical approximation that replaces the explicit trust region  with a clipped surrogate~\citep{schulman2017proximalpolicyoptimizationalgorithms}, just as PPO approximates TRPO in the single-agent case.

\subsection{ACPPO: PPO Instantiation of ACPO}
\begin{figure*}[t!]%[ht]
\begin{center}
 \centerline{\includegraphics[width=\textwidth]{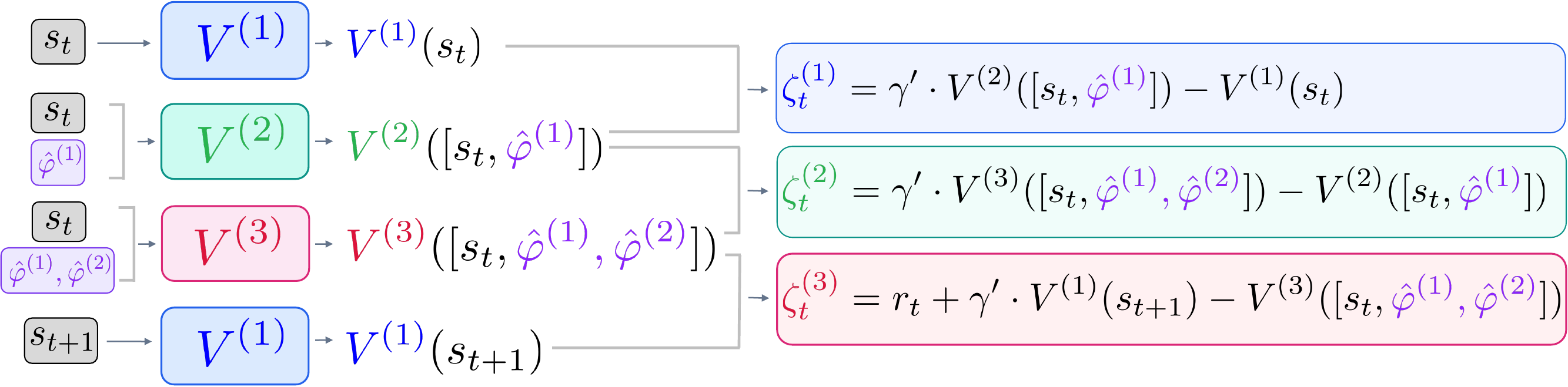}}
\caption{ Value Network Architecture for the PPO instantiation of ACPO (ACPPO)   with $N=3$ agents. Each agent's value function $V^{(i)}$ takes the state augmented with preceding belief estimates as input. TD residuals chain across  \textcolor{blue}{Agent 1},  \textcolor{green!50!black}{Agent 2}, and \textcolor{red}{Agent 3}. }
\label{figure:value_network_architecture}
\end{center}\vspace{-0.8cm}
\end{figure*}

\paragraph{Agent-Chained TD Residuals.}
Following the Bellman operators in Definition~\ref{def:bellman_operators}, the Temporal Difference~(TD) residuals are
$$
\begin{aligned}
\zeta_t^{(i)}
&=\gamma' V^{(i+1)}\left([s_{t}, b_t^{(i+1)}]\right) - V^{(i)}\left([s_t, b_t^{(i)}]\right),   \quad \forall i\in\{1, \dots, N-1 \},\\
\zeta_t^{(N)}
&=
R\left([s_t, b_t^{(N)}], a^{(N)}_t\right)
+\gamma' V^{(1)}(s_{t+1})  - V^{(N)}\left([s_t, b_t^{(N)}]\right) \\
&\approx
r_t
+\gamma' V^{(1)}(s_{t+1})
- V^{(N)}\left([s_t, b_t^{(N)}]\right),
\end{aligned}
$$
where the reward $R$ is zero for any agent $i \in \{1, \dots, N-1\}$ during micro-steps (see Figure~\ref{figure:value_network_architecture}). For intermediate agents $i < N$, the residual $\zeta_t^{(i)}$ measures how agent $i$'s action shifts the value for the next agent in the chain. Only the last agent's residual $\zeta_t^{(N)}$ incorporates the environment reward and bootstraps to $V^{(1)}$ at the next timestep.

\paragraph{Agent-Chained Generalized Advantage Estimation.}
The advantage is an exponentially-weighted sum over the chained TD residuals,
$$  A_{t}^{(i)}([s_t, b^{(i)}_t], a^{(i)}_t)=
\sum_{j=i}^N  (\gamma' \lambda')^{j-i} \zeta^{(j)}_t +
\sum_{k=1
}^\infty \sum_{j=1}^N (\gamma' \lambda')^{k N+j-i} \zeta_{t+k}^{(j)},
$$
with the full derivation provided in Appendix~\ref{appendix:advantage_computation}. The first sum aggregates TD residuals within the current timestep $t$ from agent $i$ through agent $N$. The second sum extends across future timesteps, cycling through all $N$ agents at each step.

\paragraph{PPO Objective}
Using the advantage estimates, the PPO objective can be written as a variant of policy gradient with a clipped probability ratio.

However, the score function in Eq.~\ref{eq:lifted_pg} is over the lifted action space $\Delta(\mathcal{A}^{(i)})$, where the importance ratios required by trust-region and clipped-surrogate optimizers are not well-defined. We instead optimize a surrogate objective defined over $\mathcal{A}^{(i)}$.

\begin{equation}
\label{eq:approx_pg_objective}
 \nabla_\theta \tilde J^{\text{AC}}(\vec\pi_\theta)
= \mathbb{E}_{[s,b^{(i)}] \sim d_{\gamma'}^{\vec \pi},\, a^{(i)} \sim \pi^{(i)}_\theta(\cdot \mid s, b^{(i)})}
\!\left[\nabla_\theta \log \pi^{(i)}_\theta(a^{(i)} \mid s, b^{(i)}) \cdot Q^{(i)}([s, b^{(i)}], a^{(i)})\right],
\end{equation}
    where the full derivation is provided in Appendix~\ref{appendix:final_objective_justification}.
The PPO objective can now be written with well-defined importance sampling ratios:
\begin{equation}
\label{eq:ppo_objective}
\begin{aligned}
    \mathcal{L}^{(i)}(\theta)=
    \mathbb{E}
    _{a^{(i)}_t \sim \pi_{\theta_{old}}^{(i)}(\cdot \mid s_{t}, b^{(i)}_t)}[
    \min( w^{(i)}(s_{t},b^{(i)}_t,  a^{(i)}_t)
A_{t}^{(i)},
\text{      clip}\left( w^{(i)}(s_{t},b^{(i)}_t,  a^{(i)}_t) ,
1\pm\epsilon
\right)A_{t}^{(i)} )
]
\end{aligned}
\end{equation}

where $ w^{(i)}(s_{t},b^{(i)}_t,  a^{(i)}_t):=
\pi_{\theta}^{(i)}(a^{(i)}_t | s_{t}, b^{(i)}_t) / \pi_{\theta_{old}}^{(i)}(a^{(i)}_t | s_{t}, b^{(i)}_t)$.
\paragraph{Per-Agent Importance Sampling Ratio.}
The importance sampling ratio $w^{(i)}$ in Eq.~\ref{eq:ppo_objective} is naturally defined per agent and does not require a product over all agents' policies. In contrast, alternating policy update approaches such as HAPPO and HATRPO use a ratio of the form
$w^{(i)}_{\text{BR}}(s, \vec{a}):= \prod_{j=1}^N\pi^{(j)}_{\theta}(a^{(j)} | s) /  \prod_{j=1}^N\pi^{(j)}_{\theta_{old}}(a^{(j)} | s),$
whose variance scales exponentially with the number of agents~\citep{wang2021dop}. MAPPO requires the same ratio in principle but ignores the product, resulting in a biased PPO objective.

\subsection{Belief Approximation for Practical Implementations}
\label{section:practical_implementation}
\begin{figure*}[t!]%[ht]
\begin{center}
 \centerline{\includegraphics[width=\textwidth]{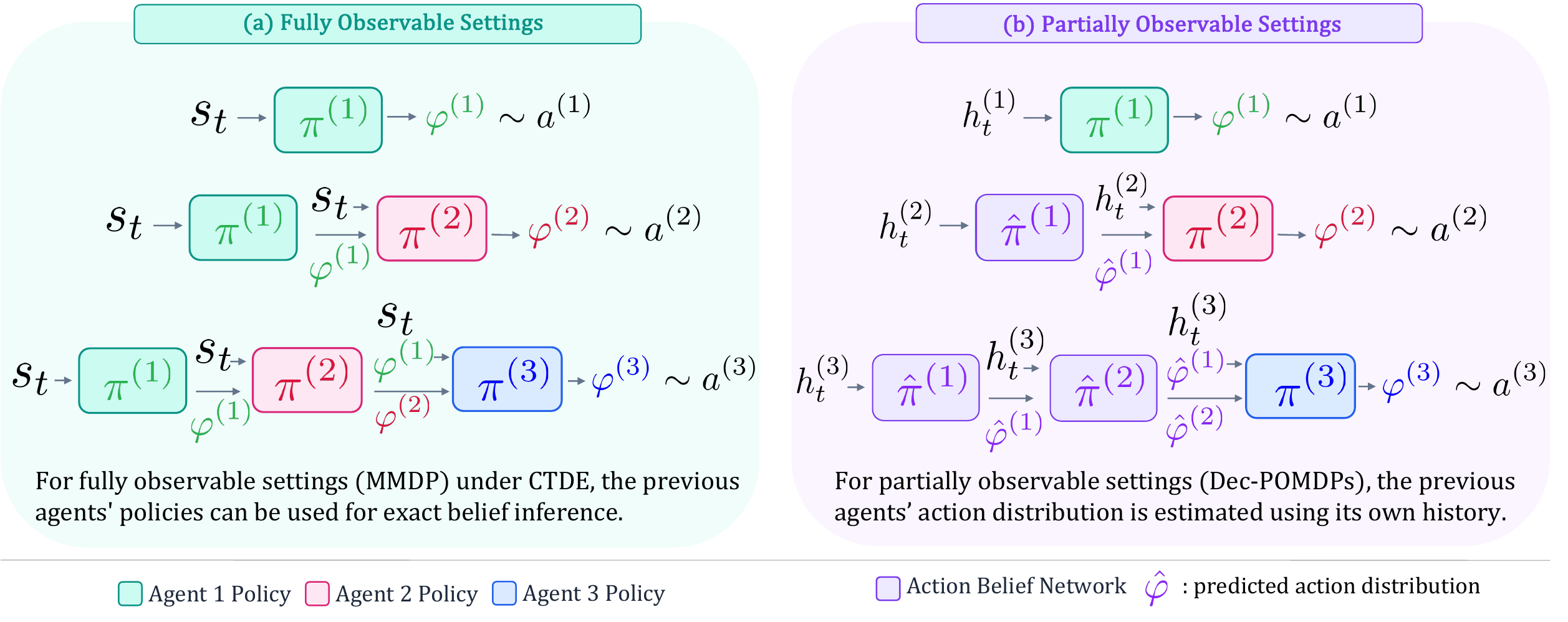}}
% \vspace{-0.5cm}
\caption{Policy architecture for ACPO under (a) fully observable settings (MMDP) and (b) partially observable settings (Dec-POMDP). In the MMDP case, $\vec\varphi^{<i}_t$ is computed exactly by querying preceding agents' networks with the shared state $s_t$. In the Dec-POMDP case, each agent learns an action belief network $\hat\pi^{(j)}$ to approximate $\varphi^{(j)}_t$ from its own history.}
\label{figure:practical_policy_architecture}
\end{center}
\vspace{-0.8cm}
\end{figure*}
The AC-BMDP belief $b^{(i)}_t$ is fully determined by $\vec\varphi^{<i}_t = [\varphi^{(1)}_t, \dots, \varphi^{(i-1)}_t]$, a distribution over preceding actions \emph{within each environment step}. Computing $\vec\varphi^{<i}_t$ therefore reduces to evaluating the preceding agents' policy networks at the current step. The procedure differs depending on whether the environment is fully or partially observable (Figure~\ref{figure:practical_policy_architecture}).

\paragraph{MMDP (Fully Observable)}
All agents share the same state $s_t$, so the preceding agents' action distributions are obtained exactly by querying their policy networks. With shared parameters and agent ID as an additional input, $\varphi^{(j)}_t = \pi_\theta(s_t, \varphi^{(1)}_t, \dots, \varphi^{(j-1)}_t;\, j)$ for each $j < i$. With separate parameters, agent $i$ maintains copies of $\pi^{(j)}_{\theta^{(j)}}$ for $j < i$, accessible during centralized training under CTDE, and queries them directly with $s_t$. In both cases, $\vec\varphi^{<i}_t$ is computed exactly.

\paragraph{Dec-POMDP (Partially Observable).}
In Dec-POMDPs, each agent conditions on its own action-observation history $h^{(i)}_t = \langle \vec o^{(i)}_{\leq t}, \vec a^{(i)}_{< t} \rangle$, and the preceding agents' action distributions cannot be computed exactly. We apply agent modelling~\citep{ALBRECHT2018modelling, papoudakis2021agent}~\citep[Section 9.6]{marl-book}, where each agent maintains an action belief network $\hat\pi^{(j)}$ that predicts $\vec\varphi^{<i}_t$ from its own history,
\[
\hat\varphi^{(j)}_t = \hat\pi^{(j)}(h^{(i)}_t), \quad j < i.
\]
The network is trained by minimizing $D_{KL}(\varphi^{(j)}_t \| \hat\varphi^{(j)}_t)$ against the centralized target $\varphi^{(j)}_t = \pi^{(j)}(h^{(j)}_t)$, available during centralized training. 
Algorithm~\ref{pseudocode:seperate_parameter} in Appendix~\ref{appendix:pseudocodes} provides the full procedure.

\paragraph{Relation to Previous Work in Opponent Modelling}

Our work is related to the concept of \emph{interactive states}~\citep{doshi2005ipomdp}, which augments the state with a model of other agents. Such opponent models often face the problem of \emph{infinite recursive reasoning}, where agent $i$ must model agent $j$'s model of agent $i$, and so on. Public belief states~\citep{brown2020rebel, publicbeliefmdp2013nayyar, foerster2019bayesian} break this recursion by introducing a common belief state shared by all agents. The AC-BMDP avoids it differently. Serialization makes the reasoning dependencies \emph{acyclic}, where agent $i$ only requires the action distributions of preceding agents $j < i$ and never models future agents. As a result, recursive belief hierarchies do not arise.

Other Deep MARL approaches typically treat opponent modelling as an auxiliary algorithmic module. LIAM~\citep{papoudakis2021agent}, for instance, augments policy inputs with a representation trained to predict the private histories of all other agents $\vec h^{-i}$. Specifically, LIAM's decoder is trained against the sampled observations and realized actions of the modelled agents as reconstruction targets. The realized actions are a higher-variance learning signal than the underlying action distribution because actions are stochastically sampled from the behavior policy. In ACPO, the action belief networks are trained directly against the action distributions of \emph{preceding} agents, which is a deterministic quantity required by the chained Bellman operator.

\begin{figure*}[t]%[ht]
\begin{center}
\centerline{\includegraphics[width=0.9\textwidth]{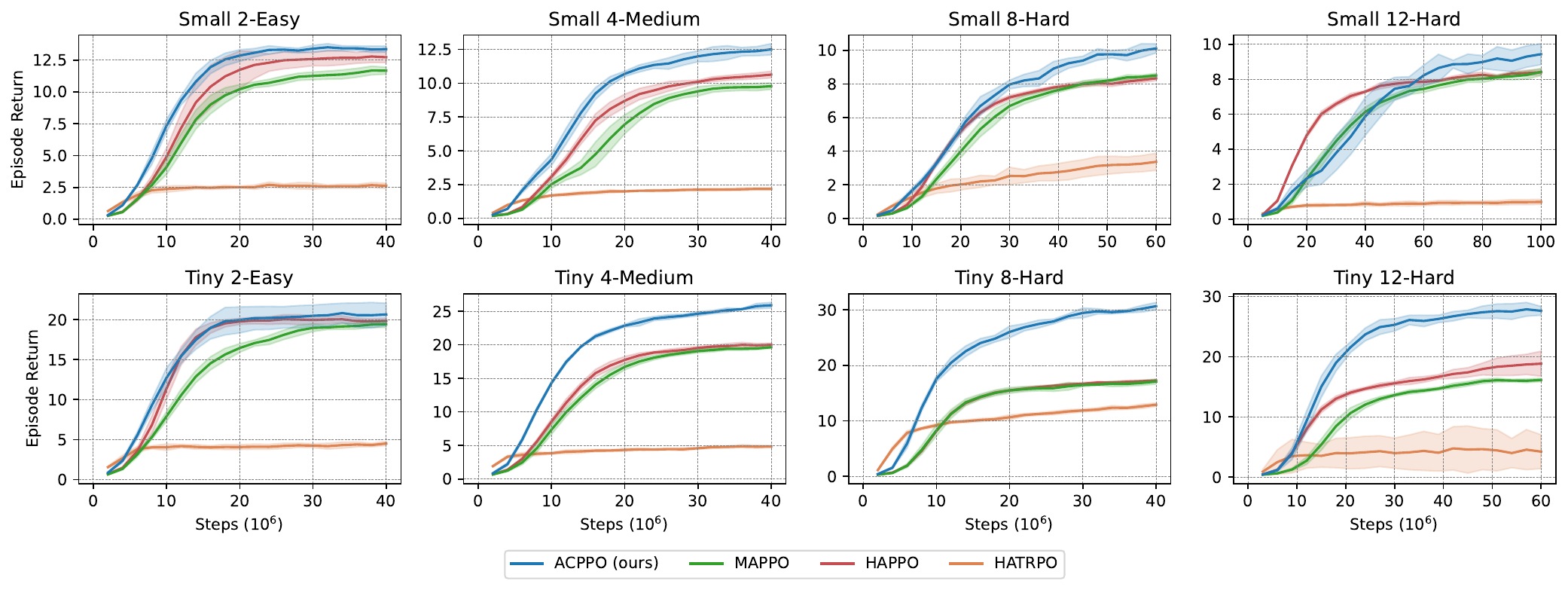}}
% \vspace{-0.3cm}
\caption{Return for Multi-Robot Warehouse~(RWARE) where return is the number of items collected and delivered successfully. {The mean and standard error over 10 seeds are reported for all tasks.} The performance advantage of ACPPO grows significantly as the number of agents increase and the map becomes smaller, where agents need higher levels of coordination being crowded in a tight space.}

\label{fig:return_curve_rware}
\end{center}
\vspace{-0.8cm}
\end{figure*}

\begin{table*}[t!]
\vspace{5pt}
\centering
\scriptsize{
\begin{tabular}{c|c|c|c|c|c}
\toprule
& {VDN} & {QMIX} & {HAPPO} & {MAPPO} & {ACPPO (Ours)}\\
\midrule
protoss\_5\_vs\_5 & $16.20 \pm 0.49$ & $\mathbf{16.53 \pm 0.55}$ & $15.93 \pm 0.54$ & $\mathbf{17.03 \pm 0.92}$ & $\mathbf{16.98 \pm 0.47}$\\
zerg\_5\_vs\_5 & $11.77 \pm 0.41$ & $\mathbf{14.33 \pm 0.63}$ & $11.81 \pm 0.63$ & $11.84 \pm 0.80$ & $12.82 \pm 1.39$\\
protoss\_10\_vs\_11 & $\mathbf{14.74 \pm 0.50}$ & $\mathbf{14.53 \pm 1.06}$ & $13.39 \pm 0.50$ & $14.57 \pm 0.33$ & $\mathbf{15.23 \pm 0.30}$\\
terran\_10\_vs\_11 & $11.59 \pm 0.67$ & $\mathbf{13.50 \pm 0.73}$ & $10.57 \pm 0.58$ & $12.03 \pm 0.50$ & $\mathbf{13.98 \pm 1.43}$\\
zerg\_10\_vs\_11 & $13.38 \pm 0.63$ & $\mathbf{14.61 \pm 0.66}$ & $10.77 \pm 0.35$ & $12.48 \pm 0.52$ & $\mathbf{13.52 \pm 1.18}$\\
\bottomrule
\end{tabular}
}
\caption{Mean return and standard error over 5 seeds on SMACv2. 
The corresponding training learning curves are provided in Figure~\ref{fig:return_curve_smacv2}.}
\label{table:smacv2_results}
\vskip -0.1in
\end{table*}

\begin{figure*}[t]%[ht]
\begin{center}
 \centerline{\includegraphics[width=\textwidth]{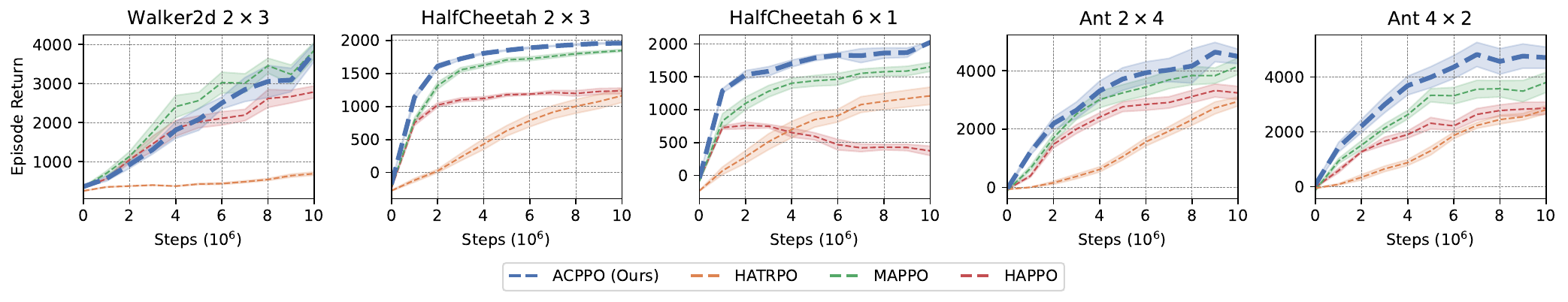}}
% \vspace{-0.5cm}
\caption{Comparison of On-Policy Algorithms on MA-MuJoCo~(Gymnasium). In environment names such as Ant 2$\times$4, the first number denotes the number of agents, while the second number indicates the action dimension per agent.}
\label{fig:return_curve_mamujoco_on_policy}
\end{center}
\vspace{-0.8cm}
\end{figure*}

\section{Experimental Results}

\paragraph{Environments}
We focus our empirical evaluation on Multi-Robot Warehouse~(RWARE)~\citep{papoudakis2021epymarl} which simulates a real-world warehouse environment consisting of multiple robots picking up requested shelves and returning them to a designated location. The main challenge in RWARE is \textit{coordination} where the agents must avoid collisions and maximize the number of shelves successfully delivered. We also evaluate our approach on StarCraft Multi-Agent Challenge v2~(SMACv2)~\citep{ellis2023smacv2} and Multi-Agent MuJoCo~\citep{peng2021facmac}, which are popular benchmarks in cooperative MARL with discrete and continuous action spaces, respectively. For SMACv2 and MA-MuJoCo, we closely follow the experimental setup in \citet{zhong2024haml}.

\paragraph{Baselines}
Our main baselines for ACPPO are MAPPO, representing independent policy optimization, and HAPPO/HATRPO, representing alternating policy optimization\footnote{We do not compare HATRPO on SMACv2 as it is estimated to take up to 40 days for training, which exceeds our computational budget for this work.  Moreover, the results from \citet{zhong2024haml} showed HATRPO is a weaker baseline in SMACv2 in comparison to HAPPO, MAPPO and QMIX. Further details are provided in Appendix~\ref{appendix:computational_resources}.}. MAPPO and HAPPO are the current state-of-the-art on-policy methods in the three domains we consider. Following the baselines considered in \citet{zhong2024haml}, we also compare against value decomposition methods including VDN~\citep{sunehag2018vdn} and QMIX~\citep{rashid2018qmix}, which are off-policy value-based methods for discrete action spaces which shows strong performance in SMACv2.

For a fair comparison, we use the code\footnote{Our code is publicly available at \url{https://github.com/dematsunaga/agent-chained-policy-optimization}.}
for all baselines provided in MARLLib~\citep{hu2023marllib}, with the same PPO backbone. For all baselines, we use the reported hyperparameters from 
\citet{papoudakis2021epymarl} for RWARE, from \citet{ellis2023smacv2} for SMACv2 and from \citet{zhong2024haml} for MA-MuJoCo.  We only tune appropriate values when the code failed to reproduce the reported performance. For MAPPO, we found the reported hyperparameters to be sufficient for reproducing the results.
For ACPPO, the same hyperparameters as MAPPO are used for all experiments in order to isolate the effect of agent-chaining, and do not conduct any additional hyperparameter tuning specific to ACPPO. Furthermore, we set the total number of parameters used by the policy to be similar between ACPPO (policy + action belief network) and the baselines (policy).
Further details on the policy network as well as hyperparameters are provided in Appendix~\ref{appendix:hyperparameter_details}.

\paragraph{Comparative Evaluation}
Our main results in Figure~\ref{fig:return_curve_rware} show that ACPPO outperforms all baselines\footnote{Note that we also compared with QMIX~\citep{rashid2018qmix} on RWARE. However, as QMIX failed to learn any meaningful behavior, we do not report their full results. This is consistent with the results reported in \citet{papoudakis2021epymarl}, and shows the limitations of value decomposition assumptions in complex coordination tasks.} on all tasks in RWARE, despite having the same PPO backbone and the same hyperparameters as MAPPO. We also see that the gap widens substantially as the number of agents increases, where the widest gap is seen in 8-agent and 12-agent domains. This provides evidence that ACPPO performs substantially better when the environment requires higher levels of coordination. Intuitively, the 12-agent maps require the most coordination among agents since it is the scenario with the most agents crowded in a tight space. 
Thus, the performance gap jumps even further for the tiny map.

In the results for MA-MuJoCo (Figures~\ref{fig:return_curve_mamujoco_on_policy}) and SMACv2 (Table~\ref{table:smacv2_results}), ACPPO is on par with or outperforms all baselines on all tasks. Crucially, ACPPO outperforms or matches MAPPO on all tasks with the same hyperparameters, which demonstrates the benefit of agent-chaining. For SMACv2, ACPPO outperforms all on-policy baselines, MAPPO and HAPPO. ACPPO is also the only on-policy algorithm competitive with QMIX. 

Finally, we also provide additional experimental results in Appendix~\ref{appendix:off_policy_results} for an off-policy instantiation of ACPO called ACTD3, which incorporates agent-chaining into TD3~\citep{DBLP:conf/icml/FujimotoHM18}. The algorithm is described in Appendix~\ref{appendix:ac_td3}.
The performance of ACTD3 is compared against strong off-policy baselines for continuous control, including MADDPG~\citep{lowe2017maddpg} a simultaneous policy optimization method and HATD3~\citep{zhong2024haml} an alternating policy optimization method.
The results for MA-MuJoCo (Figure~\ref{fig:return_curve_mamujoco_off_policy}) show that ACTD3 consistently outperforms all baseline methods, and shows that agent-chaining can be generally applied to off-policy algorithms as well.

\paragraph{Ablation Results}
\begin{figure}[t]
  \centering
    \centering
    \includegraphics[width=\textwidth]{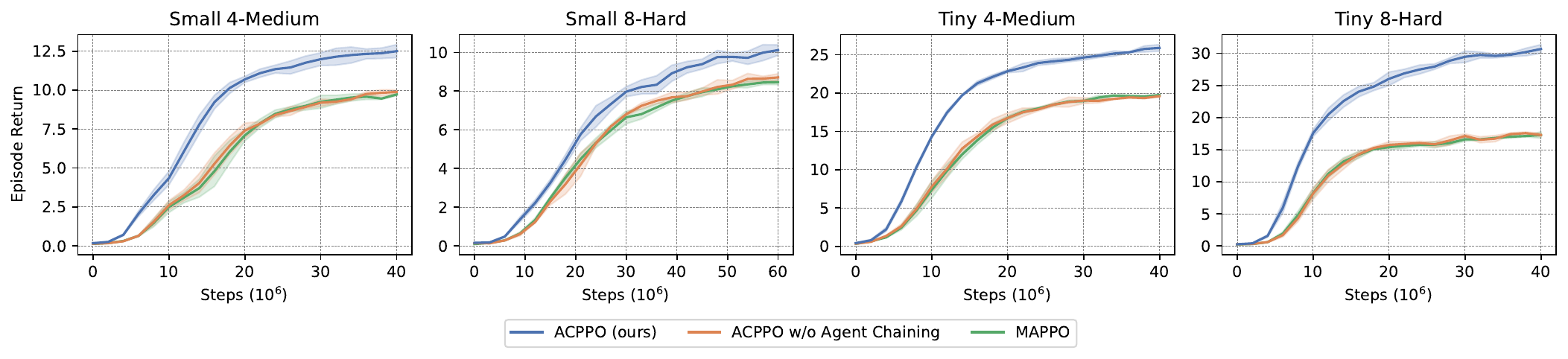}
    \caption{Ablation Results for Agent-Chaining}
    \label{fig:ablation}
\end{figure}
We ablate the core component of ACPPO, which is the advantage computation based on agent chaining. As shown in Figure~\ref{fig:ablation}, the variant ACPPO without agent chaining can be interpreted as MAPPO augmented with belief states as additional policy inputs. The performance of this variant remains close to MAPPO, indicating that the observed gains of ACPPO are not attributable to the extra input, but rather to the agent-chained advantage computation itself.

%%%%%%%%%%%%%%%%%%%%%%%%%%%%%%%%%%%%%%%%%%%%%%%%%%%%%%%%%%%%%%%%
%% Conclusion
%%%%%%%%%%%%%%%%%%%%%%%%%%%%%%%%%%%%%%%%%%%%%%%%%%%%%%%%%%%%%%%%
\section{Conclusion}
We introduced an exact decomposition of the cooperative MMDP policy gradient into a sum of per-agent terms, each computable from a per-agent score function and a decentralized critic, with no structural assumption on the joint value function. The decomposition is enabled by the Agent-Chained Belief MDP, which lifts the simultaneous joint decision into a serialized one where each agent conditions on a belief over the actions of preceding agents. We instantiated this as ACPO with on-policy and off-policy variants, and showed empirically that it outperforms strong baselines on Multi-Robot Warehouse, SMACv2, and MA-MuJoCo, with the gap widening as the number of agents grows.

\paragraph{Limitations and Future Work} 
\label{limitations}
First, the belief computation scales linearly with the number of agents. Each agent $i$ requires $i-1$ forward passes to compute $\vec{\varphi}^{<i}$, and the last agent effectively requires $N-1$ sequential forward passes. Hence the total number of forward passes across all agents grows as $\mathcal{O}(N^2)$. While the wall-clock overhead remains modest in our experiments (Appendix~\ref{appendix:runtime_statistics}), this scaling could become a bottleneck for domains with very large numbers of agents. A promising avenue for future work is to distill the autoregressive belief computation into a single feedforward network, where each agent directly predicts $\vec{\varphi}^{<i}$ from the state (or observation history) without sequential queries, or to incorporate certain symmetries among agents~\citep{pol2022multiagent, Grupen_Selman_Lee_2022}. Second, our work follows the standard setup in CTDE and focus on MMDPs for simplicity.  
% The correct form of the policy gradient for Dec-POMDPs under CTDE remains open: existing theorems are established under centralized control via occupancy states~\citep{Peralez2025simultaneousmove}, while practical CTDE methods rely on surrogates such as state-based centralized critics that are known to be biased under partial observability~\citep{lyu2023oncentralizedcritics}. 
Our formulation can be extended to Dec-POMDPs by combining beliefs over previous agent actions with occupancy states~\citep{Peralez2025simultaneousmove} or public beliefs~\citep{publicbeliefmdp2013nayyar, foerster2019bayesian}, which incorporate uncertainty over private observations. We consider this to be an interesting future direction.

%%%%%%%%%%%%%%%%%%%%%%%%%%%%%%%%%%%%%%%%%%%%%%%%%%%%%%%%%%%%%%%%
%% Appendices
%%%%%%%%%%%%%%%%%%%%%%%%%%%%%%%%%%%%%%%%%%%%%%%%%%%%%%%%%%%%%%%%
% \appendix

% \section{The first appendix}
% \label{sec:appendix1}
% This is an example of an appendix. 

% \noindent \textbf{Note:} Appendices appear before the references and are viewed as part of the ``main text'' and are subject to the 8--12 page limit, are peer reviewed, and can contain content central to the claims of the paper. 

% \section{The second appendix}
% \label{sec:appendix2}
% This is an example of a second appendix. If there is only a single section in the appendix, you may simply call it ``Appendix'' as follows:

% \section*{Appendix}
% % No label, since this can't be referenced meaningfully with \ref{}.
% This format should only be used if there is a single appendix (unlike in this document).

%%%%%%%%%%%%%%%%%%%%%%%%%%%%%%%%%%%%%%%%%%%%%%%%%%%%%%%%%%%%%%%%
%% Acknowledgements
%%%%%%%%%%%%%%%%%%%%%%%%%%%%%%%%%%%%%%%%%%%%%%%%%%%%%%%%%%%%%%%%
\subsubsection*{Acknowledgments}
\label{sec:ack}
This work was supported by the Institute of Information \&
Communications Technology Planning \& Evaluation (IITP)
grant funded by the Korea government (MSIT) (No. RS-2020-II200940, No. RS-2019-II190075, No. RS-2024-00343989, and No. RS-2024-00457882)
as well as the Artificial Intelligence Graduate School Program (Yonsei University) (RS-2020-II201361),
and by a grant from the Institute for AI and Social Innovation at Yonsei University.
Pascal Poupart also acknowledges the Canada CIFAR AI Chair
program. 

The authors would also like to thank Hongchan Jeon for helpful discussions during the earlier stages of this research, Namho Koh for providing feedback on the manuscript, and the anonymous reviewers for providing detailed constructive feedback.

% Use unnumbered third level headings for the acknowledgments. All acknowledgments, including those to funding agencies, go at the end of the paper. Only add this information once your submission is accepted and deanonymized. The acknowledgments do not count towards the 8--12 page limit.

%%%%%%%%%%%%%%%%%%%%%%%%%%%%%%%%%%%%%%%%%%%%%%%%%%%%%%%%%%%%%%%%
%% NOTE: THIS MARKS THE END OF THE "MAIN TEXT"
%%%%%%%%%%%%%%%%%%%%%%%%%%%%%%%%%%%%%%%%%%%%%%%%%%%%%%%%%%%%%%%%

%%%%%%%%%%%%%%%%%%%%%%%%%%%%%%%%%%%%%%%%%%%%%%%%%%%%%%%%%%%%%%%%
%% Bibliography
%%%%%%%%%%%%%%%%%%%%%%%%%%%%%%%%%%%%%%%%%%%%%%%%%%%%%%%%%%%%%%%%
\bibliography{main}
\bibliographystyle{rlj}

%%%%%%%%%%%%%%%%%%%%%%%%%%%%%%%%%%%%%%%%%%%%%%%%%%%%%%%%%%%%%%%%
% AUTHOR: If your paper has no supplementary materials, you may 
%         comment out the line below, which creates the title for
%         the supplementary materials.
%%%%%%%%%%%%%%%%%%%%%%%%%%%%%%%%%%%%%%%%%%%%%%%%%%%%%%%%%%%%%%%%
\appendix
\beginSupplementaryMaterials

\section{Limitations of Previous Methods}
\label{appendix:limitations_independent_policy_updates}

\begin{table}[h!]
\centering
\begin{tabular}{c|cc}
      & A   & B    \\
\hline
A     & 0   & 1    \\
B     & 1   & 2  \\
\end{tabular}
\caption{A simple \(2\times2\) matrix game where value decomposition holds and independent policy optimization methods can converge to the optimum \((B, B)\). }
\label{appendix:value-decomposition-matrix-game}
\end{table}

\paragraph{Decomposable Matrix Games}
In this matrix game, the rewards can be additively decomposed.
The individual utility functions can be decomposed as follows: $r^{(1)}(A)=r^{(2)}(A)=0$ and $r^{(1)}(B)=r^{(2)}(B)=1$.

As described in Section~\ref{subsection:previous_approaches}, independent policy optimization methods such as MAPPO optimize
\[
A^{(i)}(s, a^{(i)}) := \mathbb{E}_{\pi^{(i)}, \vec{\pi}^{-i}} \left[ A\bigl(s, a^{(i)}, \vec{a}^{-i}\bigr) \right]
\]
for each agent separately, treating \(\vec{\pi}^{-i}\) as fixed at the current policy. 

Suppose the joint policy is initialized at the suboptimal outcome \(AA\), which yields reward \(0\). For agent~1, deviating to \(B\) while agent~2 keeps playing \(A\) leads to \(R(BA) = 1 > R(AA) = 0\), so its independent update pushes probability mass toward \(B\). By symmetry, agent~2 also prefers deviating from \(A\) to \(B\), since \(R(AB) = 1 > R(AA) = 0\). When both agents update simultaneously, the joint policy moves toward \(BB\), which attains the global optimum with reward \(2\). In this decomposable setting, each agent’s local improvement direction is aligned with improvement of the joint policy.

\begin{table}[h!]
\centering
\begin{tabular}{c|cc}
      & A   & B    \\
\hline
A     & 0   & 1    \\
B     & 1   & -10  \\
\end{tabular}
\caption{A simple \(2\times2\) matrix game where independent policy optimization can converge to the catastrophic joint action \((B,B)\). }
\label{appendix:independent-matrix-game}
\end{table}

\paragraph{ Matrix Game motivating alternating policy optimization}
Now consider the matrix game in Table~\ref{appendix:independent-matrix-game}, and suppose the joint policy is initialized to  select \(AA\), yielding a reward of \(0\).
For agent~1, deviating to \(B\) while agent~2 plays \(A\) (i.e., the joint action \(BA\)) gives a higher reward than \(AA\), so its independent update drives it toward playing \(B\).
The same holds symmetrically for agent~2, since \(AB\) yields a higher reward than \(AA\).
Because both agents update independently and simultaneously, the joint policy can move toward \(BB\), which receives a reward of \(-10\).
{
Note that further updates will move the joint policy toward AA again.
In other words, independent policy optimization methods will not converge.
}

This failure mode motivates \emph{alternating policy optimization} methods such as HAPPO, where each agent updates its policy while holding the other agents fixed, and guarantee monotonic improvement of the joint policy at each update step.
In the game in Table~\ref{appendix:independent-matrix-game}, if agent~1 updates first, it will switch to playing \(B\) deterministically, since \(BA\) is better than \(AA\).
When agent~2 subsequently updates, it prefers to keep playing \(A\), because \(BA\) is better than \(BB\).
Thus the joint policy converges to \(BA\), avoiding the catastrophic outcome \(BB\).
However, as discussed in detail in Section~\ref{subsection:limitation_of_previous_approaches}, the fixed points of alternating policy optimization methods are Nash Equilibria, which {is not necessarily} the optimal joint policy.
In the game in Table~\ref{appendix:matrix:3x3}, {all the Nash Equilibria happen to be globally optimal}, whereas this is not the case for the matrix game in Table~\ref{matrix_game_policy_values_main_text}.
In contrast, ACPI directly targets the optimal joint policy in the underlying MMDP and therefore does not rely on value decomposition assumptions or on properties of Nash equilibria.

\section{Details on Centralized Training with Decentralized Execution~(CTDE)}
\label{appendix:ctde}

In our work, we consider the standard MARL paradigm of Centralized Training Decentralized Execution~(CTDE), which allows multiple policies to be jointly trained but must be executed in a decentralized fashion. 

Here we introduce related work in three different settings that are often considered in MARL: Centralized Training Centralized Execution~(CTCE), Centralized Teacher with Decentralized Student~(CTDS)~\footnote{CTDS is also referred to as \textit{Decentralizing Centralizing Solutions}~\citep{amato2024introductioncooperativemultiagentreinforcement}.}, and CTDE.

\paragraph{Centralized Training Centralized Execution (CTCE)}
Centralized Training Centralized Execution~(CTCE) methods such as Multi-Agent Transformer~(MAT)~\citep{wen2022mat} use  a joint policy of the form $\pi(a^{(1)}, \dots ,a^{(N)}|s)$ during both training rollouts and execution. 
MAT is a centralized Transformer model defined on the joint action space, and uses a joint observation encoder and joint policies during both training and execution. While a decentralized policy version is also considered, MAT requires a joint observation encoder during both training and execution. 
In MMDPs, the CTCE setting reduces to a Factored-Action MDP~\citep{guestrin2001maplanningwithfmdps, fern2012famdpwithsymbolicdp}, which is a single-agent MDP with factored action spaces. In this case, single-agent techniques such as policy iteration and value iteration can be applied directly. 

Generally, centralized control (CTCE) is not applicable to many real-world multi-agent systems such as power grids~\cite{wang2021powernetworks}, traffic signal control~\cite{chu2020traffic}, and large-scale fleet management~\cite{lin2018fleet} due to the large joint action space and prohibitive communication costs.

\paragraph{Centralized Teacher with Decentralized Student~(CTDS)} \citep{zhao2024ctds,ye2023globaloptimalitycooperativemarl, wang2023macpf} aims to decentralize centralized solutions, by assuming that a single-agent joint policy can be used for training rollout. This joint policy is used during centralized training and distilled to decentralized policies before execution. As a single-agent problem, this assumption makes convergence to optimal policies straightforward as in the CTCE case. We can view this setting as a special case of our work where we assume further access to a joint policy during training rollouts. With this additional assumption, we can solve the serialized problem introduced in Appendix~\ref{appendix:subsection_serialization_proof} without considering beliefs. However, this line of research inherits similar weaknesses of CTCE, and cannot be applied to many real-world multi-agent systems with a massive action space or prohibitive communication costs.

\paragraph{Centralized Training Decentralized Execution~(CTDE)}
In CTDE, the policy must be decentralized (fully factorized) during both training and execution, with the policy form $\vec{\pi} = \langle \pi^{(1)}, \dots \pi^{(N)} \rangle$ where $\pi^{(i)}: S\rightarrow A^{(i)}$. This is the natural MARL paradigm we consider in our work. Algorithms for independent policy optimization methods~\citep{lowe2017maddpg, yu2022mappo}, alternating policy optimization methods~\citep{kuba2022hatrpo, zhong2024haml, liu2024hasac} as well as value decomposition methods~\citep{rashid2018qmix, zhang2021fop} all fall under CTDE.

%%%%%%%%%%%%%%%%%%%%%%%%%%%%%%%%%%%%%%%%%%%%%%%%%%%%%%%%%%%%%%%%%%%
% Exact Calculation 
%%%%%%%%%%%%%%%%%%%%%%%%%%%%%%%%%%%%%%%%%%%%%%%%%%%%%%%%%%%%%%%%%%%
\section{Exact Calculation of Policies in the Matrix Game}
\label{appendix:exact_calculation_matrix}
\begin{figure}[h!]
% \tiny{
\centering
\vspace{-0.33cm}
\begin{game}{3}{3}
& $A$ & $B$  & $C$\\
$A$ & $5$ &  $-20$ & $-20$\\
$B$ & $-20$ &  $10$ & $-20$ \\
$C$ & $-20$ &  $-20$ & $20$\\
\end{game}
\caption{3x3 Matrix Game}
\label{appendix:matrix:3x3}
% }
\end{figure}

Here we provide details on how ACPI (the tabular version of ACPO) can solve the Matrix Game provided in Table~\ref{matrix_game_policy_values_main_text} and repeated in Figure~\ref{appendix:matrix:3x3}.

Due to serialization, ACPO considers this as a 2-step game even though the underlying game is a 1-step game. 
Since we are in a simple toy setting which can be solved by policy iteration, we only consider deterministic action distributions.

{Policy evaluation for agent 1 is conducted }as follows.
$$
\begin{aligned}
    &Q^{(1)}(\varphi^{(1)}) = \gamma'\mathbb{E}_{\substack{b^{(2)}=\varphi^{(1)} \\ \varphi^{(2)} \sim \pi^{(2)}(\cdot \mid b^{(2)})}} \left[Q^{(2)}(b^{(2)}, \varphi^{(2)}) \right]
\end{aligned}
$$
When we only consider deterministic $\varphi$:
$$
\begin{aligned}
    &Q^{(1)}(\delta^{(1)}_A) = \gamma'\mathbb{E}_{\substack{a^{(2)} \sim \pi^{(2)}(\cdot \mid b^{(2)}=A)}} \left[Q^{(2)}(b^{(2)}=A, a^{(2)}) \right] \\
    &Q^{(1)}(\delta^{(1)}_B) = \gamma'\mathbb{E}_{\substack{a^{(2)} \sim \pi^{(2)}(\cdot \mid b^{(2)}=B)}} \left[Q^{(2)}(b^{(2)}=B, a^{(2)}) \right] \\
    &Q^{(1)}(\delta^{(1)}_C) = \gamma'\mathbb{E}_{\substack{a^{(2)} \sim \pi^{(2)}(\cdot \mid b^{(2)}=C)}} \left[Q^{(2)}(b^{(2)}=C, a^{(2)}) \right] \\
\end{aligned}
$$
where we have used $b^{(2)} =A$ to denote the fact that agent 2 knows with probability 1 that it is in state $A$ since it knows that $\pi^{(1)}$ chooses $A$ deterministically. 
Also, we used the notation $\delta^{(1)}_A$ to denote a particular action distribution $\varphi$ which deterministically selects $A$.

For agent 2, policy evaluation is simply the reward function given in the Matrix game:
$$
\begin{aligned}
    &Q^{(2)}(b^{(2)}=A, A)=R(A, A) = 5\\
    &\qquad \qquad \vdots\\
    &Q^{(2)}(b^{(2)}=B, B)=R(B, B) = 10\\
    &\qquad \qquad \vdots\\
    &Q^{(2)}(b^{(2)}=C, C)=R(C, C) = 20\\
\end{aligned}
$$

For policy improvement, 
$$
\begin{aligned}
    &\pi^{(2)}(b^{(2)}=A) \leftarrow \arg \max_{a^{(2)} \in \{A, B, C \}}Q^{(2)}(b^{(2)}=A, a^{(2)}) \\
    &\pi^{(2)}(b^{(2)}=B) \leftarrow \arg \max_{a^{(2)} \in \{A, B, C \}}Q^{(2)}(b^{(2)}=B, a^{(2)}) \\
    &\pi^{(2)}(b^{(2)}=C) \leftarrow \arg \max_{a^{(2)} \in \{A, B, C \}}Q^{(2)}(b^{(2)}=C, a^{(2)})\\
\end{aligned}
$$

Thus, $\pi^{(2)}$ will select $A, B, C$ given agent 1 deterministically selects $A, B, C$, respectively.

For agent 1, 
$$
\begin{aligned}
    &\pi^{(1)} \leftarrow \arg \max_{a^{(1)} \in \{A, B, C \}}\gamma'Q^{(2)}(b^{(2)}=a^{(1)}, \pi^{(2)}(b^{(2)}=a^{(1)})) \\
\end{aligned}
$$

Now, let's say we are in an adversarial starting point where the policy is initialized to deterministically select $(A, A)$. 

Initially, policy improvement for $\pi^{(1)}$ leads agent 1 to continue selecting $A$ since $\pi^{(2)}$ deterministically selects A and $R(A, A)=5$ is better than $R(B, A)=R(C,A)=-20$.
However, after the first iteration, agent 2 will update its policy to select $A, B, C$ for $b^{(2)}=A, b^{(2)}=B, b^{(2)}=C$, respectively.
Thus, agent 1 in iteration 2 will select $C$ since  $C = \arg \max_{a^{(1)}} \gamma' Q^{(2)}(b^{(2)}=a^{(1)}, \pi^{(2)})$, where $\pi^{(2)}$ is now updated.

\section{Serializing the MMDP}
\label{appendix:subsection_serialization_proof}

The serialized multi-agent problem can be described as follows:
\begin{itemize}
    % \item Player function $\rho: \mathbb{N}_{\geq 0} \rightarrow \mathcal{N}$
    \item Action space $\mathcal{A}^{(i)}$ of individual actions for some $i \in \mathcal{N}$
    \item States   $[s_{t}, \vec{a}_{t}^{<i}]$ which augments the original state space with actions selected by previous agents $\vec{a}_{t}^{<i}$. 
    \item Reward function 
    $$R([s_{t}, \vec{a}_{t}^{<i}], a_{t}^{(i)})=
    \begin{cases}
        R(s_t, \vec{a}_{t}) &\text{ if } i=N\\
        0 &\text{ otherwise }\\
    \end{cases} 
    $$
    \item Transition function  $$
    \begin{aligned}
    &T([s_{t+1}, \emptyset]\mid [s_{t}, \vec{a}_{t}^{<i}], a_{t}^{(i)})=
    T(s_{t+1} \mid s_t, \vec{a}_{t})
    \text{ if } i=N\\
    &T([s_{t}, \vec{a}_{t}^{<i+1}]\mid [s_{t}, \vec{a}_{t}^{<i}], a_{t}^{(i)}) \\ &\qquad =
    \mathbb{I}\{ [\vec{a}_{t}^{<i}, a_t^{(i)}]= \vec{a}_{t}^{<i+1}\}\text{ if } i\in \{1, \dots, N -1\} \\
    \end{aligned}
    $$
    \item Discount factor $\gamma'$ where $\gamma'=\gamma^{1/N}$
\end{itemize}

The optimal policy for this serialized problem is in fact optimal for the original MMDP as well.
\begin{restatable}{theorem}{ThmSerializationEquivalence}\citep{Peralez2025simultaneousmove}
\label{theorem:serialization_equivalence}
For every MMDP, there exists a serialized multi-agent problem, of which its optimal policy is also optimal for the underlying MMDP. 
\end{restatable}
\begin{proof}
Let $\vec{\pi} = \langle {\pi}^{(1)}, \dots, {\pi}^{(N)}\rangle$ be any policy over the serialized MMDP. 
    \begin{equation}
        \begin{aligned}
            &V^{(i)}([s, \vec{a}^{<i}]; \vec\pi) \\
            &= \mathbb{E}\left[\sum_{j=i}^N (\gamma')^{j-i} R\left(s_{0}, \vec{a}^{<j}_0, a^{(j)}_0 \right) \mid s_0 = s, \vec{a}_0^{<i} = \vec{a}^{<i}, \vec{\pi}\right] 
            +\mathbb{E}\left[\sum_{t=1} \sum_{j=1}^N (\gamma')^{tN+j-i} R\left(s_{t}, \vec{a}^{<j}_t, a^{(j)}_t \right) \mid  \vec{\pi}\right] \\
            &= \mathbb{E}\left[ (\gamma')^{N-i} R\left(s_{0}, \vec{a}^{<N}_0, a^{(N)}_0 \right) \mid s_0 = s, \vec{a}_0^{<i} = \vec{a}^{<i}, \vec{\pi}\right] 
            +\mathbb{E}\left[\sum_{t=1}  (\gamma')^{tN+N-i} R\left(s_{t}, \vec{a}^{<N}_t, a^{(N)}_t \right) \mid  \vec{\pi}\right] \\
            &= \mathbb{E}\left[ (\gamma')^{N-i}R\left(s_{0}, \vec{a}^{<N}_0, a^{(N)}_0 \right) \mid s_0 = s, \vec{a}_0^{<i} = \vec{a}^{<i}, \vec{\pi}\right] +\mathbb{E}\left[\sum_{t=1}  (\gamma')^{N-i} \gamma^t R\left(s_{t}, \vec{a}^{<N}_t, a^{(N)}_t \right) \mid  \vec{\pi}\right] \\
            &= \mathbb{E}\left[\sum_{t=0}  (\gamma')^{N-i} \gamma^t R\left(s_{t}, \vec{a}^{<N}_t, a^{(N)}_t \right) \mid  s_0 = s, \vec{a}_0^{<i} = \vec{a}^{<i},\vec{\pi}\right] \\
            &= (\gamma')^{N-i}\mathbb{E}\left[\sum_{t=0}   \gamma^t R\left(s_{t}, \vec{a}^{<N}_t, a^{(N)}_t \right) \mid s_0 = s, \vec{a}_0^{<i} = \vec{a}^{<i},  \vec{\pi}\right] \\
        \end{aligned}
    \end{equation}

$V^{(1)}(s) =(\gamma')^{N-1} \mathbb{E}\left[\sum_{t=0}   \gamma^t R\left(s_{t}, \vec{a}^{<N}_t, a^{(N)}_t \right) \mid s_0=s,  \vec{\pi}\right]$ for agent $1$.

$V^{(N)}([s, \vec{a}^{<N}]) = \mathbb{E}\left[\sum_{t=0}   \gamma^t R\left(s_{t}, \vec{a}^{<N}_t, a^{(N)}_t \right) \mid s_0 = s, \vec{a}_0^{<N} = \vec{a}^{<N}, \vec{\pi}\right]$ for agent $N$.

Thus, we have established that any policy $\vec{\pi}$ in the serialized problem will obtain the same expected value in the MMDP (times a constant factor). There is a 1-1 mapping between serialized and simultaneous policies which yield the same value.

\end{proof}

\section{Belief Update}
\label{appendix:belief_update_rule}
$$
\begin{aligned}
\tau&\left([s_{t},b_{t}^{(i)}],\varphi_t^{(i)}\right)(\vec{a}_{t}^{<i+1})\\&=\frac{1}{\eta([s_{t},b_{t}^{(i)}],\varphi_t^{(i)})}\sum_{\vec{a}_{t}^{<i}}b_{t}^{(i)}( \vec{a}_{t}^{<i})
T\left([s_{t},\vec{a}_{t}^{<i+1}]|[s_t,\vec{a}_{t}^{<i}],\varphi_t^{(i)}\right) \\
&=\frac{1}{\eta([s_{t},b_{t}^{(i)}],\varphi_t^{(i)})}\sum_{\vec{a}_{t}^{<i}}b_{t}^{(i)}( \vec{a}_{t}^{<i})
\sum_{a_t^{(i)}} \varphi_t^{(i)}(a_t^{(i)}) 
T\left([s_{t},\vec{a}_{t}^{<i+1}]|[s_t,\vec{a}_{t}^{<i}],a_{t}^{(i)}\right) \\
\end{aligned}
$$
Finally,  $\eta$ is the normalization factor defined as
$$
\begin{aligned}
&\eta([s_{t},b_{t}^{(i)}],\varphi_t^{(i)})=\sum_{\vec{a}_{t}^{<i+1}}\sum_{\vec{a}_{t}^{<i}}b_{t}^{(i)}( \vec{a}_{t}^{<i})
\sum_{a_t^{(i)}} \varphi_t^{(i)}(a_t^{(i)}) 
T\left([s_{t},\vec{a}_{t}^{<i+1}]|[s_t,\vec{a}_{t}^{<i}],a_{t}^{(i)}\right)    
\end{aligned}
$$

\section{Proofs for Policy Iteration Convergence}
\label{appendix:policy_iteration_proofs}
%==========================================================================================
% Policy Evaluation
%==========================================================================================
Here we present \textit{Agent-Chained Policy Iteration}~(ACPI), a policy iteration procedure defined on the AC-BMDP. ACPI is the tabular version of ACPO introduced in Section~\ref{section:acpo}. 
Unlike alternating policy optimization approaches~\citep{zhong2024haml} where the fixed point of policy iteration is a NE, we prove that ACPI converges to the optimal joint policy of the MMDP. 

\subsection{Agent-Chained Policy Evaluation}
\label{appendix:policy_evaluation_proof}
By repeatedly applying $\mathcal{T}^{\vec{\pi}}$, we can obtain the $Q$-values for a given joint policy $\vec{\pi}$:
\begin{restatable}{lemma}{LemmaPolicyEvaluation}
\label{lemma:policy_evaluation}(Agent-Chained Policy Evaluation)
The Agent-Chained Bellman Operators in Definition~\ref{def:bellman_operators} are a contraction mapping under the infinity norm. Thus, starting with any $\vec{Q}=\langle Q^{(1)}, \dots Q^{(N)}\rangle $ and a joint policy $\vec{\pi}=\langle \pi^{(1)}, \dots,  \pi^{(N)}\rangle$, the repeated application of $\mathcal{T}^{\vec{\pi}}$   will return a set of Q-values for each agent $\langle Q^{(1, \vec{\pi})}, \dots Q^{(N, \vec{\pi})}\rangle $ in the limit.
\end{restatable}

\begin{proof}
    First, note that we can view $\langle Q^{(1)}, \dots, Q^{(N)}\rangle $ as a 
    single Q-function with the state space further augmented by agent ID. Under this perspective, we now have a single policy denoted as $\pi$ and a corresponding value function $Q^{\pi}$, defined on the AC-BMDP.

    Since the AC-BMDP is a single-agent Belief MDP, the rest follows standard convergence results of policy evaluation~\citep{Agarwal2019ReinforcementLT}, which we include for completeness.

    For any agent $i \in \{1, \dots N-1 \}$, state $[s, b^{(i)}, i]$, action $\varphi^{(i)}$ and arbitrary Q-values $Q_1, Q_2$,
    $$
    \begin{aligned}
        &\left|\mathcal{T}^{\pi}Q_1([s, b^{(i)}, i], \varphi^{(i)} ) - \mathcal{T}^{\pi}Q_2([s, b^{(i)}, i], \varphi^{(i)} )\right| \\
        & = \left|\mathbb{E}_{\substack{[s, b^{(i+1)}, i+1] = T([s, b^{(i)}, i], \varphi^{(i)})\\\varphi^{(i+1) } \sim \pi(\cdot \mid s, b^{(i+1)}, i+1)}}
        \left[\gamma'Q_1([s, b^{(i+1)}, i+1], \varphi^{(i+1)} ) - \gamma'Q_2([s, b^{(i+1)}, i+1], \varphi^{(i+1)} )\right] \right| \\
        & \leq \gamma'\max_{\varphi^{(i+1)} }\left|Q_1([s, b^{(i+1)}, i+1], \varphi^{(i+1)} ) - Q_2([s, b^{(i+1)}, i+1], \varphi^{(i+1)} )\right| \\
        & \leq \gamma'\max_{\varphi^{(i+1)}, b^{(i+1)}, j \in \{1, \dots, N\} }\left|Q_1([s, b^{(i+1)}, j], \varphi^{(i+1)} ) - Q_2([s, b^{(i+1)}, j], \varphi^{(i+1)} )\right| \\
    \end{aligned}
    $$

    For agent $N$,
    $$
    \begin{aligned}
        &\left|\mathcal{T}^{\pi}Q_1([s, b^{(N)}, N], \varphi^{(N)} ) - \mathcal{T}^{\pi}Q_2([s, b^{(N)}, N], \varphi^{(N)} )\right| \\
        & = \left|\mathbb{E}_{\substack{[s', 1] \sim T(\cdot \mid [s, b^{(N)}, N], \varphi^{(N)})\\\varphi^{(1) } \sim \pi(\cdot \mid s', 1)}}
        \left[\gamma'Q_1([s', 1], \varphi^{(1)} ) - \gamma'Q_2([s', 1], \varphi^{(1)} )\right] \right| \\
        & \leq \gamma'\max_{s', \varphi^{(1)} }\left|Q_1([s', 1], \varphi^{(1)} ) - Q_2([s', 1], \varphi^{(1)} )\right| \\
         & \leq \gamma'\max_{s', \varphi^{(1)}, j \in \{1, \dots, N\} }\left|Q_1([s', j], \varphi^{(1)} ) - Q_2([s', j], \varphi^{(1)} )\right| \\
    \end{aligned}
    $$

    Thus, $\mathcal{T}^\pi$ is a contraction mapping under the infinity norm, i.e. there exists $\gamma' \in [0, 1)$ such that
    
    $$\| \mathcal{T}^{\pi}Q_{1} - \mathcal{T}^{\pi}Q_2\|_\infty 
    \leq \gamma' \| Q_{1} - Q_2\|_\infty
    $$
    
    Since $\mathcal{T}^{\pi}$ is a contraction mapping, we have the following:
    $$ 
    \begin{aligned}
        \| Q_k - Q^\pi\|_\infty &= \| \mathcal{T}^{\pi}Q_{k-1} - \mathcal{T}^{\pi}Q^\pi\|_\infty \\
        & \leq \gamma' \| Q_{k-1} - Q^\pi\|_\infty\\
        & \ \ \vdots\\
        & \leq ({\gamma'})^k \| Q_{0} - Q^\pi\|_\infty\\
    \end{aligned}
    $$
    If we let $k \rightarrow \infty ,  \| Q_k - Q^\pi\|_\infty =0$, and $\lim_{k\rightarrow \infty}Q_k = Q^\pi$. By the Banach fixed-point theorem, this solution is unique.
    
\end{proof}

%==========================================================================================
% Policy Improvement
%==========================================================================================

\subsection{Agent-Chained Policy Improvement}
\label{appendix:policy_improvement_proof}
During policy improvement, each agent's policy $\pi^{(i)}$ is updated to select the greedy action distribution with respect to its own $Q^{(i)}$:
\begin{equation}
    \label{eq:policy_improvement}
    \begin{aligned}
        \pi^{(i)}_{new}([s, b^{(i)}]) \leftarrow \arg \max_{\varphi^{(i)}} Q^{(i,\vec{\pi})}([s, b^{(i)}], \varphi^{(i)}) \qquad \forall i, b^{(i)}, s .
    \end{aligned}
\end{equation}

\begin{restatable}{lemma}{LemmaPolicyImprovement}
\label{lemma:policy_improvement} (Agent-Chained Policy Improvement)
Given a policy $\vec{\pi}=\langle \pi^{(1)}, \dots, \pi^{(N)}\rangle$, let $Q^{(i,\vec{\pi})}$ denote the $i$-th agent's value function for a joint policy $\vec{\pi}$. If we update the new policy $\vec{\pi}_{new}=\langle \pi^{(1)}_{new}, \dots, \pi^{(N))}_{new}\rangle$ by Eq.~\ref{eq:policy_improvement},
then
\[
Q^{(i, \vec{\pi}_{new})} ([s, b^{(i)}], \varphi^{(i)})\geq Q^{(i, \vec{\pi})} ([s, b^{(i)}], \varphi^{(i)})
\]
\end{restatable}

\begin{proof}
    As in the proof for Lemma~\ref{lemma:policy_evaluation}, we consider $\vec{\pi}$ to be a single policy $\pi$ which is augmented by agent ID in the state space.
    
    For any $i \in \{1, \dots N-1 \}, s, b^{(i)}, \varphi^{(i)}$,
    $$
    \begin{aligned}
        Q^{{\pi}}([s, b^{(i)}, i], \varphi^{(i)}) &= \gamma' \mathbb{E}_{\substack{[s, b^{(i+1)}, i+1]= T([s, b^{(i)}, i], \varphi^{(i)}) \\ \varphi^{(i+1)} \sim \pi(\cdot \mid [s, b^{(i+1)}, i+1])}}\left[Q^{{\pi}}\left([s, b^{(i+1)}, i+1], \varphi^{(i+1)}\right) \right] \\
        & \leq \gamma' \mathbb{E}_{\substack{[s, b^{(i+1)}, i+1]= T([s, b^{(i)}, i], \varphi^{(i)})}}\left[\max_{\varphi^{(i+1)}}Q^{{\pi}}\left([s, b^{(i+1)}, i+1], \varphi^{(i+1)}\right) \right] \\
        & = \gamma' \mathbb{E}_{\substack{[s, b^{(i+1)}, i+1]= T([s, b^{(i)}, i], \varphi^{(i)}) \\ \varphi^{(i+1)} \sim \pi_{new}(\cdot \mid [s, b^{(i+1)}, i+1])}}\left[Q^{{\pi}}\left([s, b^{(i+1)}, i+1], \varphi^{(i+1)}\right) \right]
    \end{aligned}
    $$

    For $i=N$ and any $s, b^{(i)}, \varphi^{(i)}$,
    $$
    \begin{aligned}
        Q^{{\pi}}&([s, b^{(N)}, N], \varphi^{(N)})\\ &= R([s, b^{(N)}, N], \varphi^{(N)})+\gamma' \mathbb{E}_{\substack{[s', 1]= T([s, b^{(N)}, N], \varphi^{(N)}) \\ \varphi^{(1)} \sim \pi(\cdot \mid [s', 1])}}\left[Q^{{\pi}}\left([s',1], \varphi^{(1)}\right) \right] \\
        & \leq R([s, b^{(N)}, N], \varphi^{(N)})+ \gamma' \mathbb{E}_{\substack{[s', 1]= T([s, b^{(N)}, N], \varphi^{(N)}))}}\left[\max_{\varphi^{(1)}}Q^{{\pi}}\left([s',1], \varphi^{(1)}\right) \right] \\
        & = R([s, b^{(N)}, N], \varphi^{(N)})+ \gamma' \mathbb{E}_{\substack{[s', 1]= T([s, b^{(N)}, N], \varphi^{(N)}) \\ \varphi^{(1)} \sim \pi_{new}(\cdot \mid [s', 1])}}\left[Q^{{\pi}}\left([s',1], \varphi^{(1)}\right) \right] \\
    \end{aligned}
    $$
Thus, for any $i \in \{1, \dots N \}, s, b^{(i)}, \varphi^{(i)}$,
    $$
    \begin{aligned}
        Q^{{\pi}}([s, b^{(i)}, i], \varphi^{(i)}) 
        &\leq R([s, b^{(i)}, i], \varphi^{(i)})\\&\qquad+\gamma' \mathbb{E}_{\substack{[s, b^{(i+1)}, i+1]= T([s, b^{(i)}, i], \varphi^{(i)}) \\ \varphi^{(i+1)} \sim \pi_{new}(\cdot \mid [s, b^{(i+1)}, i+1])}}\left[Q^{{\pi}}\left([s, b^{(i+1)}, i+1], \varphi^{(i+1)}\right) \right] \\
        &\ \ \vdots \\
        &\leq Q^{{\pi_{new}}}([s, b^{(i)}, i], \varphi^{(i)}) 
    \end{aligned}
    $$
\end{proof}
Crucially, since each agent's improvement in Eq.~\ref{eq:policy_improvement} only depends on its own $Q^{(i, \vec{\pi})}$, all agents can update their policies independently. This is unlike MA-PI~\citep{zhong2024haml}, which requires alternating updates where $\pi^{(i)}$ can only be improved after $\vec{\pi}^{<i}$ has been updated.
%==========================================================================================
% Convexity of Q proof
%==========================================================================================
\subsection{Characterization of $Q^{(i, *)}$}
\label{appendix:convexity_of_Q_proof}

Here we provide a useful property of the optimal Q-values of the AC-BMDP. 
\begin{restatable}{theorem}{CorollaryMaxQPhiEqualsMaxQA}
    \label{theorem:max_q_phi_equals_max_q_a}
    For all $i\in \{1, \dots, N\}, s, b^{(i)}$,
    $$
    \max_{\varphi^{(i)}}Q^{(i, *)}\left([s, b^{(i)}], \varphi^{(i)}\right)=\max_{a^{(i)}}Q^{(i, *)}\left([s, b^{(i)}], \delta_{a^{(i)}}\right)
    $$
\end{restatable}

\begin{proof}
    
    We prove the claim by induction.

    For any $s, {b}^{(N)}, \varphi^{(N)}$ at terminal timesteps,
    $$
    \begin{aligned}
        Q^{(N)}([s, b^{(N)}], \delta_{a^{(N)}}) 
        &= R([s, b^{(N)}], a^{(N)}) \\
    \end{aligned}
    $$
    
    $$
    \begin{aligned}
        Q^{(N)}([s, b^{(N)}], \varphi^{(N)}) 
        &= R([s, b^{(N)}], \varphi^{(N)})\\
        &=  \sum_{a^{(N)}}\varphi^{(N)}(a^{(N)})R([s, b^{(N)}], a^{(N)})\\
        % &=\sum_{\vec{a}^{<N}} b^{(N)}(\vec{a}^{<N}) \sum_{a^{(N)}}\varphi^{(N)}(a^{(N)})R(s, \vec{a}^{<N}, a^{(N)})\\
        &= \sum_{a^{(N)}}\varphi^{(N)}(a^{(N)})Q^{(N)}([s, b^{(N)}], \delta_{a^{(N)}})\\
    \end{aligned}
    $$

    Also for any $\varphi^{(j)}$,
    $$
    \begin{aligned}
        Q^{(N)}([s, \varphi^{(j)}, \vec \varphi^{-j}], \varphi^{(N)}) 
        &= R([s, \varphi^{(j)}, \vec \varphi^{-j}], \varphi^{(N)})\\
        &=  \sum_{{a}^{(j)}}\varphi^{(j)}({a}^{(j)})R([s, a^{(j)}, \vec \varphi^{-j}], \varphi^{(N)})\\
        &=  \sum_{{a}^{(j)}}\varphi^{(j)}({a}^{(j)})Q^{(N)}([s, a^{(j)}, \vec \varphi^{-j}], \varphi^{(N)})\\
    \end{aligned}
    $$
    Since $Q^{(N)}$ is affine in the action distribution, the statement holds for the base case.
    
    For the inductive step, 
    $$
    \begin{aligned}
        \max_{\varphi^{(i)}} Q^{(i,*)} ([s, b^{(i)}], \varphi^{(i)}) &= \gamma' \max_{\varphi^{(i)}}\max_{\varphi^{(i+1)}} Q^{(i+1, *)} ([s, b^{(i)}, \varphi^{(i)}], \varphi^{(i+1)}) \\ 
        &= \gamma' \max_{\varphi^{(i+1)}}\max_{\varphi^{(i)}} Q^{(i+1, *)} ([s, b^{(i)}, \varphi^{(i)}], \varphi^{(i+1)}) \\
        &= \gamma' \max_{\varphi^{(i+1)}}\max_{a^{(i)}} Q^{(i+1, *)} ([s, b^{(i)}, a^{(i)}], \varphi^{(i+1)}) \\
        &= \max_{a^{(i)}} Q^{(i+1, *)}([s, b^{(i)}], a^{(i)}),\\
    \end{aligned}
    $$

    For any $j\neq i$,
    $$
    \begin{aligned}
        \max_{\varphi^{(j)}} Q^{(i,*)} ([s, \varphi^{(j)}, \vec \varphi^{-j}], \varphi^{(i)}) &= \gamma' \max_{\varphi^{(j)}}\max_{\varphi^{(i+1)}} Q^{(i+1, *)} ([s, \varphi^{(j)}, \vec \varphi^{-j}, \varphi^{(i)}], \varphi^{(i+1)}) \\ 
        &= \gamma' \max_{\varphi^{(i+1)}} \max_{\varphi^{(j)}} Q^{(i+1, *)} ([s, \varphi^{(j)}, \vec \varphi^{-j}, \varphi^{(i)}], \varphi^{(i+1)}) \\ 
        &= \gamma' \max_{\varphi^{(i+1)}} \max_{a^{(j)}} Q^{(i+1, *)} ([s, a^{(j)}, \vec \varphi^{-j}, \varphi^{(i)}], \varphi^{(i+1)}) \\ 
        &= \max_{a^{(j)}} Q^{(i,*)} ([s, a^{(j)}, \vec \varphi^{-j}], \varphi^{(i)})\\
    \end{aligned}
    $$ 
\end{proof}
%==========================================================================================
% Policy Iteration Convergence
%==========================================================================================

When restricted to deterministic action distributions $\delta_{a^{(i)}}$ , the beliefs become deterministic and the AC-BMDP reduces to a serialized version of the MMDP. Theorem~\ref{theorem:max_q_phi_equals_max_q_a} shows that since the optimal Q-values coincide between the AC-BMDP and the MMDP, the optimal policies coincide as well.

% \CorollaryMaxQPhiEqualsMaxQA*
\subsection{Agent-Chained Policy Iteration}
\label{appendix:policy_iteration_proof}
Alternating between Agent-Chained Policy Evaluation and Policy Improvement provably converges to the optimal policy of the MMDP. 

\begin{restatable}{theorem}{ThmPolicyIterationConvergence}
\label{theorem:pi_convergence} (Agent-Chained Policy Iteration)
Starting from any policy $\vec{\pi} \in \Pi$, the sequence of value functions $\vec{Q}^{\vec{\pi}_n} $ and the improved policies $\vec{\pi}_{n+1}$ converges to the optimal value functions and the policy of the AC-BMDP, i.e, $$Q^{(i, *)}([s, b^{(i)}], \varphi^{(i)})=\lim_{n \rightarrow \infty} Q^{(i,\vec{\pi}_n)}([s, b^{(i)}], \varphi^{(i)}) \geq Q^{(i,\vec{\pi})}([s, b^{(i)}], \varphi^{(i)})$$ for any $\vec{\pi}, i, s, b^{(i)}, \varphi^{(i)}$.
Furthermore the optimal policy of the AC-BMDP is also optimal in the underlying MMDP.
\end{restatable}

\begin{proof}

By the monotonic improvement property in Lemma~\ref{lemma:policy_improvement}, we know that for any $i \in \{1, \dots, N \}, b^{(i)},s, \varphi^{(i)} $,  
$$
Q^{\vec{\pi}_{n+1}}([s, b^{(i)}], \varphi^{(i)}) \geq Q^{\vec{\pi}_n}([s, b^{(i)}], \varphi^{(i)}) 
$$.

If there is no improvement,
$$
\begin{aligned}
    Q^{\vec{\pi}_{n}}([s, b^{(i)}], \varphi^{(i)}) &= Q^{\vec{\pi}_{n+1}}([s, b^{(i)}], \varphi^{(i)}) \\
    &=  \gamma' Q^{\vec{\pi}_{n+1}}([s, b^{(i+1)}], {\pi}^{(i+1)}_{n+1}([s, b^{(i+1)}]))\\
    &=  \gamma' Q^{\vec{\pi}_{n}}([s, b^{(i+1)}], {\pi}^{(i+1)}_{n+1}([s, b^{(i+1)}]))\\
    &=  \gamma' \max_{\varphi^{(i+1)}}Q^{\vec{\pi}_{n}}([s, b^{(i+1)}], \varphi^{(i+1)}))\\
\end{aligned}
$$
where $[s, b^{(i+1)}] = T([s, b^{(i)}], \varphi^{(i)})$. Thus, at the limit $\lim_{n \rightarrow \infty} Q^{\vec{\pi}_{n}}([s, b^{(i)}], \varphi^{(i)})$, the Bellman optimality equations are satisfied.

Due to Theorem~\ref{theorem:max_q_phi_equals_max_q_a}, it is sufficient to consider the following policy improvement procedure considering only the $\delta_{a^{(i)}}$, which is the space of deterministic $\varphi^{(i)}$:

\begin{equation}
    \label{eq:policy_improvement_delta}
    \begin{aligned}
        \forall i, b^{(i)}, s, \pi^{(i)}_{new}([s, b^{(i)}]) \leftarrow \arg \max_{a^{(i)}} Q^{\vec{\pi}}([s, b^{(i)}], \delta_{a^{(i)}}) 
    \end{aligned}
\end{equation}

% Note that if we restrict ourselves to an AC-BMDP defined over the space of deterministic 1-step policies $\delta_{a^{(i)}}$, all of the components in the AC-BMDP are equivalent to that of the serialized version of the MMDP. 
\end{proof}

While the policy improvement in Eq.~\eqref{eq:policy_improvement} is defined over the space of action distributions $\varphi^{(i)} \in \Delta(\mathcal{A}^{(i)})$, there is no loss of generality in restricting to deterministic action distributions $\delta_{a^{(i)}}$~(Theorem~\ref{theorem:max_q_phi_equals_max_q_a}). When $\varphi^{(i)}$ is deterministic for all agents, the belief $b^{(i)}$ is also deterministic, and the AC-BMDP reduces to a serialized version of the MMDP. The full pseudocode is provided in Algorithm~\ref{pseudocode:agent_chained_policy_iteration} in Appendix~\ref{appendix:pseudocodes}.

\section{Proofs for Section~\ref{section:acpo}}
\label{appendix:pg-decomposition-proofs}

% For advantages defined on the action space $a^{(i)}$,

% \[
% \begin{aligned}
%     A^{(i)}(s, b^{(i)}, a^{(i)})&=Q^{(i)}(s, b^{(i)}, a^{(i)})-V^{(i)}(s,b^{(i)}) \\
%     &=  Q^{(i)}(s, b^{(i)}, a^{(i)})
%     - \mathbb{E}_{\tilde a^{(i)} \sim \pi^{(i)}(\cdot \mid s, b^{(i)})}\left[Q^{(i)}(s,b^{(i)}, \tilde a^{(i)})\right]
% \end{aligned}
% \] 

We prove our central result: the MMDP policy gradient admits an exact
decomposition into a sum of $N$ per-agent terms, each involving only a
per-agent score function and a per-agent decentralized critic
$Q^{(i)}$ from the AC-BMDP. The decomposition is exact and requires no
structural assumption on the joint Q-function.

\subsection{Proof for the Multi-Agent Policy Gradient Decomposition Theorem}
\label{appendix:decomposition-proof}

\ThmMapgDecomposition*

\begin{proof}
The proof has two steps: split the AC-BMDP gradient by agent index, then convert from the AC-BMDP to MMDP gradient via a value-equivalence constant.

\paragraph{Step 1: Per-agent split of the AC-BMDP gradient.}
By Eq.~\ref{eq:lifted_pg}, the AC-BMDP policy gradient is
\[
\nabla_\theta J^{\text{AC}}(\vec\pi_\theta)
= \mathbb{E}_{[s,b^{(i)}] \sim d^{\text{AC}}_{\gamma'},\, \varphi^{(i)} \sim \pi^{(i)}_\theta(\cdot|s,b^{(i)})}
\!\left[\nabla_\theta \log \pi^{(i)}_\theta(\varphi^{(i)}|s,b^{(i)}) \cdot Q^{(i)}([s,b^{(i)}], \varphi^{(i)})\right].
\]
The AC-BMDP visitation $d^{\text{AC}}_{\gamma'}$ is defined over the augmented state $[s, b, i]$, where the agent index $i$ cycles deterministically through $1, \dots, N$ across micro-steps. Only micro-steps $\tau = tN + (i-1)$ have agent index $i$, so the visitation factors as
\begin{equation}
d^{\text{AC}}_{\gamma'}([s, b^{(i)}, i]) = (\gamma')^{i-1} \sum_{t \geq 0} \gamma^t \Pr\!\left(s_t = s,\, b^{(i)}_t = b^{(i)}\mid \vec\pi_\theta\right).
\label{eq:per-agent-visitation}
\end{equation}
Splitting the AC-BMDP expectation along the deterministic agent-index dimension gives
\begin{equation}
\nabla_\theta J^{\text{AC}}(\vec\pi_\theta)
= \sum_{i=1}^N \mathbb{E}_{[s,b^{(i)}] \sim d^{\text{AC}}_{\gamma'}(\cdot, \cdot, i),\, \varphi^{(i)} \sim \pi^{(i)}_\theta}
\!\left[\nabla_\theta \log \pi^{(i)}_\theta(\varphi^{(i)}|s,b^{(i)}) \cdot Q^{(i)}([s,b^{(i)}], \varphi^{(i)})\right].
\label{eq:per-agent-acbmdp}
\end{equation}

\paragraph{Step 2: Value equivalence.}

Let $\vec \mu_\theta$ be defined as the policy defined on the MMDP, induced from the policy from AC-BMDP $\vec \pi_\theta$: 
\[
\vec \mu_\theta (\vec a \mid s):= \mathbb{E}_{\varphi^{(1)} \sim \pi^{(1)}(\cdot \mid s), \varphi^{(2)} \sim \pi^{(2)}(\cdot \mid s, \varphi^{(1)}), \dots } \left[\prod_{i=1}^N  \varphi^{(i)}(a^{(i)})\right].
\]
The AC-BMDP provides non-zero reward only at agent $N$'s micro-step, where it equals the MMDP reward $r_t$. Starting from the initial state $[s_0, \emptyset, 1]$,
\[
J^{\text{AC}}(\vec\pi_\theta)
= \mathbb{E}\!\left[\sum_{\tau \geq 0} (\gamma')^\tau R^{\text{AC}}_\tau\right]
= \mathbb{E}\!\left[\sum_{t \geq 0} (\gamma')^{tN + (N-1)} r_t\right]
= (\gamma')^{N-1} J(\vec\mu_\theta),
\]
since $(\gamma')^N = \gamma$. The constant $(\gamma')^{N-1}$ is independent of $\theta$, so $\nabla_\theta J^{\text{AC}} = (\gamma')^{N-1} \nabla_\theta J$. Substituting into \eqref{eq:per-agent-acbmdp} and dividing by $(\gamma')^{N-1}$ gives \eqref{eq:pg-decomposition}.
\end{proof}

\paragraph{Remark on the leading constant.}
The factor $1/(\gamma')^{N-1}$ in \eqref{eq:pg-decomposition} is a positive scalar independent of $\theta$ and is absorbed into the effective step size of any first-order or trust-region optimizer. It does not affect the gradient direction or the fixed point of the resulting optimization.

\section{Agent-Chained Proximal Policy Optimization~(ACPPO)}
\subsection{Derivation for the Policy Gradient Objective defined in Eq.~\ref{eq:approx_pg_objective}}

\label{appendix:final_objective_justification}

% \PropositionAgentChainedPolicyGradient*
$$
\begin{aligned}
    \nabla_{\theta} J^{AC}&(\theta; s, b^{(i)}) 
    \\ &= \mathbb{E}_{\varphi^{(i)} \sim \pi^{(i)}_{\theta}(\cdot|s, b^{(i)})} \left[ \nabla_{\theta} \log \pi^{(i)}_{\theta}(\varphi^{(i)}|s, b^{(i)}) Q^{(i)}([s, b^{(i)}], \varphi^{(i)}) \right] \\
    &= \int_{\varphi^{(i)}} \pi^{(i)}_{\theta}(\varphi^{(i)}|s, b^{(i)}) \nabla_\theta \log \pi^{(i)}_\theta(\varphi^{(i)}|s, b^{(i)}) Q^{(i)}([s, b^{(i)}], \varphi^{(i)}) d \varphi^{(i)} \\
    &{=}\int_{\varphi^{(i)}} \pi^{(i)}_\theta(\varphi^{(i)}|s, b^{(i)}) \nabla_\theta \log \pi^{(i)}_\theta(\varphi^{(i)}|s, b^{(i)}) \int_{a^{(i)}} \varphi^{(i)}(a^{(i)}) Q^{(i)}([s, b^{(i)}], a^{(i)}) d a^{(i)} d \varphi^{(i)} \\
    &= \int_{a^{(i)}} \int_{\varphi^{(i)}} \pi^{(i)}_\theta(\varphi^{(i)}|s, b^{(i)}) \nabla_\theta \log \pi^{(i)}_\theta(\varphi^{(i)}|s, b^{(i)}) \varphi^{(i)}(a^{(i)}) d \varphi^{(i)} ~ Q^{(i)}([s, b^{(i)}], a^{(i)}) d a^{(i)}  \\
    &= \int_{a^{(i)}} \int_{\varphi^{(i)}} \nabla_{\theta} \pi^{(i)}_\theta(\varphi^{(i)} | s, b^{(i)}) \varphi^{(i)}(a) d \varphi^{(i)} ~ Q^{(i)}([s, b^{(i)}], a^{(i)}) d a^{(i)}  \\
    &= \int_{a^{(i)}} \int_{\varphi^{(i)}} \nabla_{\theta} \pi^{(i)}_\theta(\varphi^{(i)} | s, b^{(i)}) \Pr(a^{(i)} | \varphi^{(i)}) d \varphi^{(i)} ~ Q^{(i)}([s, b^{(i)}], a^{(i)}) d a^{(i)}  \\
    &= \int_{a^{(i)}} \nabla_{\theta} \left( \underbrace{\int_{\varphi^{(i)}} \pi^{(i)}_\theta(\varphi^{(i)} | s, b^{(i)}) \Pr(a^{(i)} \mid \varphi^{(i)}) d \varphi^{(i)}}_{= \pi^{(i)}_\theta(a^{(i)}|s, b^{(i)})} \right) ~ Q^{(i)}([s, b^{(i)}], a^{(i)}) d a^{(i)} \\
    &= \int_{a^{(i)}} \nabla_{\theta} \pi^{(i)}_\theta(a^{(i)}|s, b^{(i)})  Q^{(i)}([s, b^{(i)}], a^{(i)}) d a^{(i)} \\
    &= \int_{a^{(i)}} \pi^{(i)}_\theta(a^{(i)}|s, b^{(i)})  \nabla_{\theta}  \log \pi^{(i)}_\theta(a^{(i)}|s, b^{(i)})  Q^{(i)}([s, b^{(i)}], a^{(i)}) d a^{(i)} \\
    &= \mathbb{E}_{a^{(i)} \sim \pi^{(i)}_\theta(a^{(i)}|s, b^{(i)})} \left[ \nabla_{\theta}  \log \pi^{(i)}_\theta(a^{(i)}|s, b^{(i)})  Q^{(i)}([s, b^{(i)}], a^{(i)}) \right]. \\
\end{aligned}
$$

The policy structure is now in the familiar form $\pi^{(i)}(a^{(i)} \mid s, b^{(i)})$ which stochastically outputs actions $a^{(i)}$. This standard policy structure allows for well-defined importance sampling ratios in the PPO objective.

\subsubsection{Note on the 3rd step}
\begin{equation}
Q^{(i)}\big([s,\bel{i}],\vphi^{(i)}\big)
\;\overset{?}{=}\;
\E_{a^{(i)} \sim \vphi^{(i)}}\Big[ Q^{(i)}\big([s,\bel{i}], a^{(i)}\big) \Big],
\label{eq:affine-claim}
\end{equation}
i.e., the claim that $Q^{(i)}$ is \emph{affine} in the action distribution 

While \eqref{eq:affine-claim} does not hold in general, it is easy to see that it holds with equality for $i=N$.

For agents $i < N$, we show instead that Eq.~\ref{eq:approx_pg_objective} is a surrogate objective that is equal under detached belief states.

Recall the lifted policy gradient on the AC-BMDP,
\begin{equation}
\nabla_\theta J^{\mathrm{AC}}(\vec{\pi}_\theta)
= \E_{[s,\bel{i}] \sim \dgp,\; \vphi^{(i)} \sim \pi^{(i)}_\theta(\cdot \mid s, \bel{i})}
\Big[ \nabla_\theta \log \pi^{(i)}_\theta\big(\vphi^{(i)} \mid s, \bel{i}\big)\, Q^{(i)}\big([s,\bel{i}], \vphi^{(i)}\big) \Big],
\label{eq:lifted-pg}
\end{equation}

We first note that under CTDE the belief is a deterministic function of the state and the \emph{same} parameter vector,
\[
\bel{i} \equiv \vec{\vphi}^{<i}_\theta(s) = \big[\pi^{(1)}_\theta(s),\, \pi^{(2)}_\theta(s, \bel{2}),\, \ldots\big],
\]
so $\theta$ influences the objective both through agent $i$'s conditional policy \emph{and} through agent $i$'s input. Eq.~\ref{eq:approx_pg_objective} keeps only the first pathway. This is made explicit with a stop-gradient operator $\sg[\cdot]$.

\begin{definition}[Belief-detached surrogate]
\label{def:surrogate}
\begin{equation}
\tilde{J}^{AC}(\theta)
\;\coloneqq\;
\sum_{i=1}^{N}
\E_{[s,\bel{i}] \sim \dgp}
\E_{a^{(i)} \sim \pi^{(i)}_\theta(\cdot \mid s,\, \sg[\bel{i}])}
\Big[ Q^{(i)}\big([s, \sg[\bel{i}]], a^{(i)}\big) \Big],
\end{equation}
where $Q^{(i)}([s,b],a)$ denotes the chained critic evaluated at a realized action $a$ with the belief input held at its sampled (behavior-policy) value.
\end{definition}

\paragraph{Characterizing the omitted term}

We now state what separates $\nabla_\theta \tilde{J}$ from the true joint gradient.

\begin{equation}
\nabla_\theta \log \vec{\pi}_\theta(\va \mid s)
= \sum_{i=1}^{N}
\left.\partial_\theta \log \pi^{(i)}_\theta\big(a^{(i)} \mid s, \bel{i}\big)\right|_{\bel{i}\ \mathrm{fixed}}
\;+\;
\sum_{i=2}^{N}
\Big(\partial_\theta \vec{\vphi}^{<i}_\theta(s)\Big)^{\!\top} \nabla_{b}\, \log \pi^{(i)}_\theta\big(a^{(i)} \mid s, \bel{i}\big) .
\label{eq:chainrule}
\end{equation}
Consequently
\begin{equation}
\nabla_\theta J(\vec{\pi}_\theta)
= \nabla_\theta \tilde{J}^{AC}(\theta)
+ \E_{s \sim \dgp}\sum_{i=2}^{N}
\Big(\partial_\theta \vec{\vphi}^{<i}_\theta(s)\Big)^{\!\top}
\E_{\va \sim \vec{\pi}_\theta}\Big[ \nabla_b \log \pi^{(i)}_\theta\big(a^{(i)} \mid s, \bel{i}\big)\; A\big(s,\va\big) \Big].
\label{eq:correction}
\end{equation}

Overall, the surrogate objective $\tilde J^{AC}$ in Eq.~\ref{eq:approx_pg_objective} is always  exact for $i=N$ since Eq.~\ref{eq:affine-claim} holds with equality. For $i<N$, it is the exact gradient of a surrogate objective which detaches belief inputs.

\subsection{Details on Advantage Computation for Agent-Chained Proximal Policy Optimization~(ACPPO)}
\label{appendix:advantage_computation}
The advantage can be written as an exponentially-weighted sum over the TD residuals,
$$A_{t}^{(1)}=  
\sum_{j=1}^N  (\gamma' \lambda')^{j-1} \zeta^{(j)}_t + 
\sum_{k=1
}^\infty \sum_{j=1}^N (\gamma' \lambda')^{k N+j-1} \zeta_{t+k}^{(j)}$$

$$A_{t}^{(2)}=  
\sum_{j=2}^N  (\gamma' \lambda')^{j-2} \zeta^{(j)}_t + 
\sum_{k=1
}^\infty \sum_{j=1}^N (\gamma' \lambda')^{k N+j-2} \zeta_{t+k}^{(j)}$$
$$\vdots$$
$$A_{t}^{(N)}= 
 \zeta^{(N)}_t + 
\sum_{k=1
}^\infty \sum_{j=1}^N (\gamma' \lambda')^{k N+j-N} \zeta_{t+k}^{(j)}$$

$$ \therefore A_{t}^{(i)}= 
\sum_{j=i}^N  (\gamma' \lambda')^{j-i} \zeta^{(j)}_t + 
\sum_{k=1
}^\infty \sum_{j=1}^N (\gamma' \lambda')^{k N+j-i} \zeta_{t+k}^{(j)}.$$

%==========================================================================================
% Pseudocodes
%==========================================================================================
\section{Pseudocodes}
We first present the pseudocode for ACPI which is defined directly on the action space consisting of 1-step policies $\varphi^{(i)}$. This is a straightforward policy iteration procedure defined on the AC-BMDP.

\label{appendix:pseudocodes}
\begin{algorithm}[h]
    \caption{Agent-Chained Policy Iteration~(ACPI)}
    \begin{algorithmic}[1]
        % \STATE \textbf{Input}:  $\mathcal{M}$ : FAB-MDP
            \STATE Randomly initialize $\vec{\pi}=\left({\pi}^{(1)}, \ldots, {\pi}^{(N)}\right)$ and $\vec{Q}=\left({Q}^{(1)}, \ldots, Q^{(N)}\right)$.
             % \STATE Randomly initialize 
            \WHILE{$\vec{\pi}$ not converged}
            \WHILE{$\vec{Q}$ not converged}
                \STATE \textcolor{blue}{\# Policy Evaluation}
                \STATE $$\begin{aligned}
                    \forall &i \in \{1, \dots, N-1 \}, s, b^{(i)}, \varphi^{(i)}, \\
                    &Q^{(i)}([s, b^{(i)}], \varphi^{(i)})  
                        \leftarrow \gamma' \mathbb{E}_{\substack{b^{(i+1)} = T([s, b^{(i)}], \varphi^{(i)})\\
                        \varphi^{(i+1)} \sim \pi^{(i+1)}(\cdot \mid s, b^{(i+1)})
                        }}\left[Q^{(i+1)}\left([s, b^{(i+1)}], \varphi^{(i+1)}\right) \right]
                    \end{aligned}$$
                    \STATE $$\begin{aligned}
                    \forall s, b^{(N)}, &\varphi^{(N)}, \\
                    &Q^{(N)}([s, b^{(N)}], \varphi^{(N)})  \\& \quad  \leftarrow R\left([s, b^{(N)}], \varphi^{(N)}\right)
                        + \gamma' \mathbb{E}_{\substack{s' \sim T(\cdot \mid [s, b^{(N)}], \varphi^{(N)})\\ \varphi^{(1)} \sim \pi^{(1)}(\cdot \mid s')
                        }}\left[Q^{(1)}\left(s', \varphi^{(1)}\right) \right] \\
                    \end{aligned}$$
            \ENDWHILE
                \STATE \textcolor{blue}{\# Policy Improvement}
                \STATE 
                     $$
                    \begin{aligned}
                    \forall i \in \{1,\dots,N\}&, s, b^{(i)},\\
                    &\pi^{(i)}(s, b^{(i)})  
                        \leftarrow \arg \max_{\varphi^{(i)}} Q^{(i)}\left([s, b^{(i)}], \varphi^{(i)}\right) 
                    \end{aligned}
                    $$
            \ENDWHILE
    \end{algorithmic}
    \label{pseudocode:agent_chained_policy_iteration}
\end{algorithm}

As we showed in Theorem~\ref{theorem:max_q_phi_equals_max_q_a}, the AC-BMDP has a special structure which ensures that it is sufficient to consider the (finite) space of deterministic action distributions $\varphi^{(i)}$ for finding an optimal policy. Thus, we can also define an equivalent policy iteration procedure over the action space $\delta_{a^{(i)}}$ (Algorithm~\ref{pseudocode:agent_chained_policy_iteration_delta}).

% -----------------------------------------
% in terms of delta
% -----------------------------------------
\begin{algorithm}[h]
    \caption{Agent-Chained Policy Iteration (Deterministic action distribution $\delta_{a^{(i)}}$)}
    \begin{algorithmic}[1]
        % \STATE \textbf{Input}:  $\mathcal{M}$ : FAB-MDP
            \STATE Randomly initialize $\vec{\pi}=\left({\pi}^{(1)}, \ldots, {\pi}^{(N)}\right)$ and $\vec{Q}=\left({Q}^{(1)}, \ldots, Q^{(N)}\right)$.
             % \STATE Randomly initialize 
            \WHILE{$\vec{\pi}$ not converged}
            \WHILE{$\vec{Q}$ not converged}
                \STATE \textcolor{blue}{\# Policy Evaluation}
                \STATE $$\begin{aligned}
                    \forall &i \in \{1, \dots, N-1 \}, s, \vec{a}^{<i}, a^{(i)}, \\
                    &Q^{(i)}([s, b^{(i)}=\vec{a}^{<i}], \delta_{a^{(i)}})  
                        \leftarrow \gamma' \mathbb{E}_{\substack{b^{(i+1)} = T([s, b^{(i)}=\vec{a}^{<i}], \delta_{a^{(i)}})\\\delta_{a^{(i+1)}} \sim \pi^{(i+1)}(\cdot \mid s, b^{(i+1)})
                        }}\left[Q^{(i+1)}\left([s, b^{(i+1)}], \delta_{a^{(i+1)}}\right) \right]
                    \end{aligned}$$
                    \STATE $$\begin{aligned}
                    \forall s, \vec{a}^{<N},  a^{(N)}, \\
                    &Q^{(N)}([s, b^{(N)}=\vec{a}^{<N}], \delta_{a^{(N)}})  \\& \quad  \leftarrow R\left([s, \vec{a}^{<N}], a^{(N)}\right)
                        + \gamma' \mathbb{E}_{\substack{s' \sim T(\cdot \mid [s, b^{(N)}=\vec{a}^{<N}], \delta_{a^{(N)}})\\\delta_{a^{(1)}} \sim \pi^{(1)}(\cdot \mid s')
                        }}\left[Q^{(1)}\left(s', \delta_{a^{(1)}}\right) \right] \\
                    \end{aligned}$$
            \ENDWHILE
                \STATE \textcolor{blue}{\# Policy Improvement}
                \STATE 
                     $$
                    \begin{aligned}
                    \forall i \in \{1,\dots,N\}, &s, \vec{a}^{<i},\\
                    &\pi^{(i)}(s, \vec{a}^{<i})  
                        \leftarrow \arg \max_{a^{(i)}} Q^{(i)}\left([s, \vec{a}^{<i}], \delta_{a^{(i)}}\right) 
                    \end{aligned}
                    $$
            \ENDWHILE
    \end{algorithmic}
    \label{pseudocode:agent_chained_policy_iteration_delta}
\end{algorithm}
% -----------------------------------------

Finally, we present ACPO which is a practical algorithm that aims to approximate ACPI via the PPO objective.

\begin{algorithm}[h]
    \caption{Agent-Chained Proximal Policy Optimization (ACPPO)}
    \begin{algorithmic}[1]
        \STATE \textbf{Initialize}: Actor networks ${\theta}_0 =
    [\theta_0^{(1)},\dots,\theta_0^{(N)}]$, Critic networks $\vec{\psi}_0 =
    [\psi^{(1)}_0,\dots,\psi^{(N)}_0]$, and action belief networks $\vec{\phi}_0=[\phi_0^{(2)}, \dots, \phi_0^{(N)}]$.
        
        \WHILE{$t \leq t_{max}$}
        \STATE Compute beliefs autoregressively: $b^{(i)} \leftarrow [\pi^{(1)}_\theta(s), \dots, \pi^{(i-1)}_\theta(s, b^{(i-1)})], \forall i =1, \dots N$ 
        \STATE Collect transitions $\left(s_t, \{a_t^{(i)}\}_{i=1}^N, r_t, s_{t+1}\right)$ by running the joint policy $\vec{\pi}_{{\theta}_k}$.
        
        \STATE Compute advantages after each episode: $\forall i=1, \dots,N$,
        $$  
        A_{t}^{(i)}= 
            \sum_{j=i}^N  (\gamma' \lambda')^{j-i} \zeta^{(j)}_t + 
            \sum_{k=1}^\infty \sum_{j=1}^N (\gamma' \lambda')^{k N+j-i} \zeta_{t+k}^{(j)}
        $$
        
        \STATE Update actors with the PPO-Clip objective: $\forall i=1\dots, N,$
        
        $$
\begin{aligned}
\mathbb{E}_{a^{(i)}_t \sim \pi_{\theta_{old}}^{(i)}(\cdot \mid s_{t}, b^{(i)}_t)}
[\min(& w^{(i)}(s_{t},b^{(i)}_t,  a^{(i)}_t)
A_{t}^{(i)}, \text{clip}( w^{(i)}(s_{t},b^{(i)}_t,  a^{(i)}_t), 1-\epsilon, 1+\epsilon)A_{t}^{(i)} )]
\end{aligned}
$$
        
        where 
        $ w^{(i)}(s_{t},b^{(i)}_t,  a^{(i)}_t):=
        \frac{\pi_{\theta}^{(i)}(a^{(i)}_t \mid s_{t}, b^{(i)}_t)}
        {\pi_{\theta_{old}}^{(i)}(a^{(i)}_t \mid s_{t}, b^{(i)}_t)}$.
        
        \STATE Update decentralized critics: $\forall i=1\dots, N,$
        
        $$
        \psi^{(i)}_{new}=\arg \min _{\psi^{(i)}} 
        \mathbb{E}\left[
        \left(V_{\psi^{(i)}}^{(i)}\left(s_t, b^{(i)}_t\right)-\hat{R}_t^{(i)}\right)^2
        \right]
        $$
        
        \STATE \textcolor{blue}{\# Action Belief Network Training (Optional)}
        
        \FOR{each agent $i=2,\dots,N$}
            \FOR{each preceding agent $j<i$}
            
            \STATE Compute true distribution
            $$
            \varphi^{(j)}_t = \pi_{\theta}^{(j)}(\cdot \mid s_t, b^{(j)}_t)
            $$
            
            \STATE Compute predicted distribution
            $$
            \hat{\varphi}^{(j)}_t =
            \hat{\pi}_{\phi}^{(j)}(\cdot \mid s_t, b^{(j)}_t)
            $$
            
            \STATE Update $\hat{\pi}^{(j)}_{\phi}$ by minimizing
            
            $$
            \mathcal{L}_{KL}^{(j)} =
            D_{KL}(\varphi^{(j)}_t \parallel \hat{\varphi}^{(j)}_t)
            $$
            
            \ENDFOR
        \ENDFOR
        
        \ENDWHILE
    \end{algorithmic}
    \label{pseudocode:seperate_parameter}
\end{algorithm}

%%%%%%%%%%%%%%%%%%%%%%%%%%%%%%%%%%%%%%
% td3
%%%%%%%%%%%%%%%%%%%%%%%%%%%%%%%%%%%%%%

\section{Agent-Chained Twin Delayed Deterministic Policy Gradient~(ACTD3)}

\label{appendix:ac_td3}

We instantiate the ACPO framework with the Twin Delayed DDPG~(TD3) algorithm~\citep{DBLP:conf/icml/FujimotoHM18} for continuous action spaces. In AC-TD3, each agent $i$ maintains an actor $\pi^{(i)}_\theta$ that maps its augmented state $[s, b^{(i)}]$ to a continuous action, along with a pair of twin critics $Q^{(i)}_{\psi,1}, Q^{(i)}_{\psi,2}$ that follow the agent-chained structure from Definition~\ref{def:bellman_operators}. Target networks $\bar{\pi}^{(i)}_\theta$ and $\bar{Q}^{(i)}_{\psi,1}, \bar{Q}^{(i)}_{\psi,2}$ are maintained via Polyak averaging.

\paragraph{The role of $\varphi^{(i)}$ (unnoised actions) vs. $a^{(i)}$ (noised actions).}
Since TD3 policies are deterministic, the action distribution $\varphi^{(i)} = \pi^{(i)}_\theta(s, b^{(i)})$ is simply the deterministic output of the actor network (i.e., the unnoised action). Independent Gaussian noise $\epsilon^{(i)} \sim \mathcal{N}(0, \sigma^2 I)$ is added privately for exploration:
$$a^{(i)} = \varphi^{(i)} + \epsilon^{(i)},$$
where $a^{(i)}$ is the noised action that interacts with the environment. Since the noise is private and independent across agents, only $\varphi^{(i)}$ is publicly available and used to construct the beliefs of subsequent agents: $b^{(i+1)} = [b^{(i)}, \varphi^{(i)}]$.

\paragraph{Critic objectives.}
Following the agent-chained Bellman operators (Definition~\ref{def:bellman_operators}), the critic loss for each agent $i$ uses a clipped double-Q target $\bar{Q}^{(i+1)} = \min(\bar{Q}^{(i+1)}_{\psi,1}, \bar{Q}^{(i+1)}_{\psi,2})$. The 1-step targets are:

for $i=1, \dots, N-1$,
$$\begin{aligned}
J_{Q}^{(i)}&(\psi)  = \mathbb{E}_{\substack{\left(s, b^{(i)}, \varphi^{(i)}\right) \sim \mathcal{D} \\ \varphi^{(i+1)} \sim \pi^{(i+1)}_{\theta}\left(\cdot \mid s,b^{(i)}, \varphi^{(i)}\right) }}
\left[\left(Q_{\psi}^{(i)}([s,  b^{(i)}], \varphi^{(i)})-y^{(i)}\right)^2\right]\\
&\text{s.t. }y^{(i)}=\gamma'\bar{Q}^{(i+1)}([s,  b^{(i)}, \varphi^{(i)}],{\varphi}^{(i+1)})
\end{aligned}
$$

$$\begin{aligned}
J_{Q}^{(N)}&(\psi)  = \mathbb{E}_{\substack{\left(s, b^{(N)}, \varphi^{(N)}, r, s'\right) \sim \mathcal{D} \\ {\varphi^{(1)}} \sim {\pi}_{\theta}^{(1)}\left(\cdot \mid s'\right) }}
\left[\left(Q_{\psi}^{(N)}([s,  b^{(N)}], \varphi^{(N)})-y^{(N)}\right)^2\right]\\
&\text{s.t. }y^{(N)}=r+\gamma'\bar{Q}^{(1)}(s', {\varphi}^{(1)})
\end{aligned}
$$

For practical implementations, it is often useful to consider $k$-step returns. 

$$\begin{aligned}
J_{Q}^{(i)}&(\psi)  = \mathbb{E}_{\substack{\left(s_t, b^{(i)}_t, \varphi^{(i)}_t, \{r_{t+j}\}_{j=0}^k, s_{t+k+1}\right) \sim \mathcal{D} \\ \vec{\varphi}_{t+k+1} \sim \vec{\pi}_{\theta}\left(\cdot \mid s_{t+k+1}\right) }}
\left[\left(Q_{\psi}^{(i)}([s_t,  b^{(i)}_t], \varphi^{(i)}_t)-{(\gamma')}^{N-i} y^{(i)}_t\right)^2\right]\\
&\text{s.t. }y^{(i)}_t=r_{t} + \gamma r_{t+1} + \cdots + \gamma^{k}r_{t+k} +\gamma^{k+1}Q_{\psi}^{(i+1)}([s_{t+k+1}, b^{(i+1)}_{t+k+1} ], \varphi^{(i+1)}_{t+k+1}) 
\end{aligned}
$$

We find that using $k$-step returns in this way works better in practice as each agent now has a dense reward signal in the targets (rather than only the last agent). We note that $\gamma$ denotes the discount factor in the original MMDP and $\gamma'=\gamma^{1/N}$. The ${(\gamma')}^{N-i} $ discount is to adjust the micro step to match with the last agent. For example, for agent 1, the reward given at the current timestep is ${(\gamma')}^{N-1} r_t$.

\paragraph{Actor objective and independent updates.}
Each actor is trained to maximize its own critic, with gradients flowing through the policy:

$$
\begin{aligned}
    J_{\pi}^{(i)} (\theta
    ) = \mathbb{E}_{s, b^{(i)} \sim \mathcal{D}} \left[ Q^{(i)}_{\psi,1}\left([s, b^{(i)}], \pi^{(i)}_\theta(s, b^{(i)})\right) \right]
\end{aligned}
$$

Since each agent has its own decentralized critic, all $N$ actor updates can be performed independently (unlike sequential methods such as HATD3). Following standard TD3, actor and target network updates are delayed, occurring once every $d$ critic updates.

\begin{algorithm}[h]
    \caption{Agent-Chained TD3~(ACTD3)}
    \begin{algorithmic}[1]
        \STATE \textbf{Input}: $N$ agents, discount $\gamma$, Polyak rate $\rho$, target noise $\sigma_{\text{target}}$, noise clip $c$,  $k$-step return horizon $k$
        \STATE \textbf{Initialize}: Actor networks $\pi^{(i)}_\theta$ with parameters $\theta^{(i)}$, twin critics $Q^{(i)}_{\psi,1}, Q^{(i)}_{\psi,2}$ with parameters $\psi^{(i)}$, target networks $\bar{\pi}^{(i)}_\theta \leftarrow \pi^{(i)}_\theta$, $\bar{Q}^{(i)}_{\psi} \leftarrow Q^{(i)}_{\psi}$, replay buffer $\mathcal{D}$
        \STATE Set serialized discount $\gamma' \leftarrow \gamma^{1/N}$

        \FOR{each environment step}
        \STATE \textcolor{blue}{\#  Data Collection }
        \FOR{each agent $i = 1, \dots, N$ \textbf{in parallel}}
            \STATE Compute beliefs autoregressively: $b^{(i)} \leftarrow [\pi^{(1)}_\theta(s), \dots, \pi^{(i-1)}_\theta(s, b^{(i-1)})]$ 
            \STATE Compute unnoised action: $\varphi^{(i)} \leftarrow \pi^{(i)}_\theta(s, b^{(i)})$
            \STATE Compute noised action: $a^{(i)} \leftarrow \varphi^{(i)} + \epsilon^{(i)}, \quad \epsilon^{(i)} \sim \mathcal{N}(0, \sigma^2 I)$
        \ENDFOR
        \STATE Execute joint action $\vec{a} = [a^{(1)}, \dots, a^{(N)}]$, observe $r, s'$
        \STATE Store $(s, \vec{\varphi}, \vec{a}, r, s')$ in $\mathcal{D}$

        \STATE
        \STATE \textcolor{blue}{\#  Training }
        \STATE Sample mini-batch $\{(s_t,\vec{\varphi}_t, \vec{a}_t, \{r_{t+j}\}_{j=0}^k, s_{t+k+1})\}$ from $\mathcal{D}$
        \STATE
        \STATE \textcolor{blue}{\# Critic update }
        \FOR{each agent $i = 1, \dots, N$ \textbf{in parallel}}
            \STATE Compute smoothed target action: $\tilde{\varphi}^{(i+1)} \leftarrow \text{clip}(\bar{\varphi}^{(i+1)} + \text{clip}(\epsilon, -c, c), a_{\text{low}}, a_{\text{high}})$
            \STATE Compute $k$-step target: 
            \STATE $y^{(i)}_t \leftarrow \sum_{j=0}^{k} \gamma^j r_{t+j} + \gamma^{k+1} \min_{m} \bar{Q}^{(i+1)}_{\psi, m}([s_{t+k+1}, {b}^{(i+1)}_{t+k+1}], \tilde{\varphi}^{(i+1)})$
            \STATE Scale target: $\hat{y}^{(i)}_t \leftarrow (\gamma')^{N - i} \cdot y^{(i)}_t$
            \STATE Update critics: $ \nabla_{\psi^{(i)}} \sum_{m=1}^{2} \left\| Q^{(i)}_{\psi, m}([s_t, b^{(i)}_t], \varphi^{(i)}_t) - \hat{y}^{(i)}_t \right\|^2$
        \ENDFOR

        \STATE
        \STATE \textcolor{blue}{\# Actor update (all agents in parallel)}
        \FOR{each agent $i = 1, \dots, N$ \textbf{in parallel}}
            \STATE $\nabla_{\theta^{(i)}} \mathbb{E}\!\left[ Q^{(i)}_{\psi,1}([s, b^{(i)}], \pi^{(i)}_{\theta}(s, b^{(i)})) \right]$
        \ENDFOR
        \STATE \textcolor{blue}{\# Soft update all target networks}
        \FOR{each agent $i = 1, \dots, N$}
            \STATE $\bar{\theta}^{(i)} \leftarrow \rho \theta^{(i)} + (1 - \rho)\bar{\theta}^{(i)}$,\quad $\bar{\psi}^{(i)} \leftarrow \rho \psi^{(i)} + (1 - \rho)\bar{\psi}^{(i)}$
        \ENDFOR

        \ENDFOR
    \end{algorithmic}
    \label{pseudocode:actd3}
\end{algorithm}

%%%%%%%%%%%%%%%%%%%%%%%%%%%%%%%%%%%%%%
% Off-Policy Variant
%%%%%%%%%%%%%%%%%%%%%%%%%%%%%%%%%%%%%%

%%%%%%%%%%%%%%%%%%%%%%%%%%%%%%%%%%%%%%%%%%%%%%%%%%%%%%%%%%%%%%%%%%%
% Justification for final loss
%%%%%%%%%%%%%%%%%%%%%%%%%%%%%%%%%%%%%%%%%%%%%%%%%%%%%%%%%%%%%%%%%%%
\clearpage
\newpage

\section{Learning Curve for SMACv2}

\begin{figure*}[h!]%[ht]
\begin{center}
 \centerline{\includegraphics[width=\textwidth]{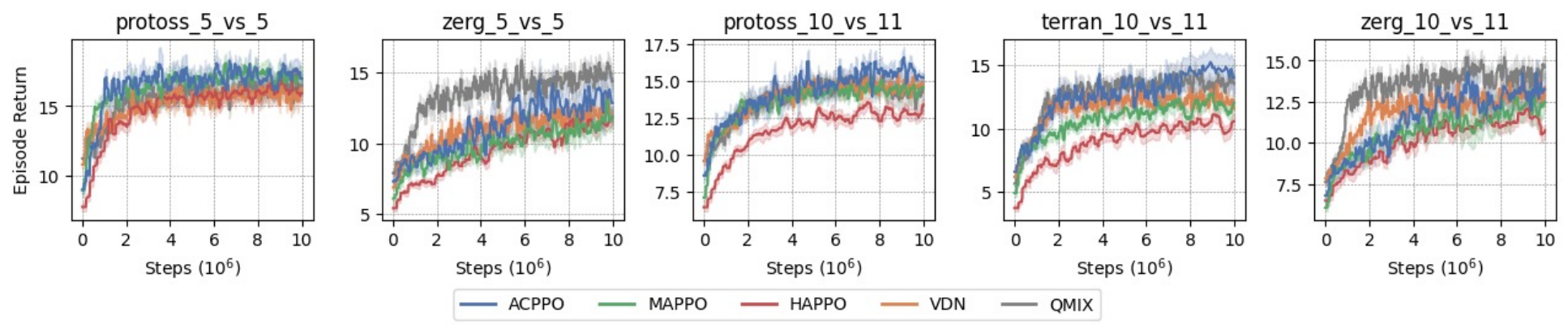}}
% \vspace{-0.5cm}
\caption{Return for SMACv2 with mean and standard error over 5 seeds.}
\label{fig:return_curve_smacv2}
\end{center}
\vspace{-0.8cm}
\end{figure*}

\section{Off-Policy Comparison}
\label{appendix:off_policy_results}
\begin{figure*}[h!]%[ht]
\begin{center}
 \centerline{\includegraphics[width=\textwidth]{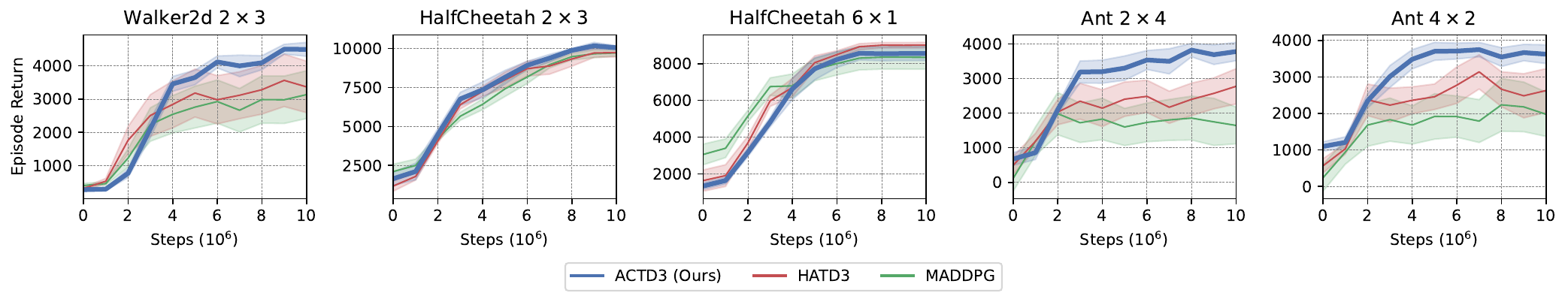}}
% \vspace{-0.5cm}
\caption{Comparison of Off-Policy Algorithms on MA-MuJoCo~(Gymnasium).}
\label{fig:return_curve_mamujoco_off_policy}
\end{center}
\vspace{-0.8cm}
\end{figure*}

%TODO Refine Explanation of performance gap
We further evaluate an off-policy variant of ACPO, referred to as ACTD3, following the on-policy experiments. ACTD3 applies the agent-chaining mechanism to TD3~\citep{DBLP:conf/icml/FujimotoHM18}. As off-policy baselines, we compare against MADDPG~\citep{lowe2017maddpg} and HATD3~\citep{zhong2024haml}. As shown in Figure~\ref{fig:return_curve_mamujoco_off_policy}, ACTD3 achieves performance comparable to or better than the baseline methods across different tasks. The performance gap is the highest for Ant $4 \times 2$ which is the most challenging domain. 

We note that in HalfCheetah tasks, there is a substantial gap in episode returns between the on-policy and off-policy experiments. Similar discrepancies have been reported in prior works~\citep{christodoulou2019softactorcriticdiscreteaction,spinningup} in single-agent experiments on MuJoCo. We attribute this discrepancy to the fundamental differences between on-policy and off-policy learning paradigms. On-policy methods update policies using trajectories generated by the current policy, whereas off-policy methods learn from transitions sampled from a replay buffer containing experiences collected by past policies. The ability to reuse past experiences enables off-policy methods to perform multiple gradient updates per environment interaction, which often results in better sample complexity. 

Finally, we note that MA-MuJoCo~(Gymnasium) is a more recent benchmark in comparison to MA-MuJoCo~(Gym) where the underlying physics engine uses Gymnasium/MuJoCo-v5  and Gym/MuJoCo-v2, respectively. As previous work such as \citet{zhong2024haml} used the older version of MA-MuJoCo~(Gym), the results from their paper cannot be directly compared to ours.

%%%%%%%%%%%%%%%%%%%%%%%%%%%%%%%%%%%%%%
% details on practical implementation
%%%%%%%%%%%%%%%%%%%%%%%%%%%%%%%%%%%%%%

%%%%%%%%%%%%%%%%%%%%%%%%%%%%%%%%%%%%%%
% RUNTIME vs RETURN
%%%%%%%%%%%%%%%%%%%%%%%%%%%%%%%%%%%%%%

\section{Wall-Clock Training Time }
\label{appendix:runtime_statistics}

\paragraph{Runtime Statistics}
% Figure~\red{TODO} shows the training time results for ACPO as well as the baselines. It
\begin{figure}[h]
  \centering
    \centering
    \includegraphics[width=0.5\linewidth]{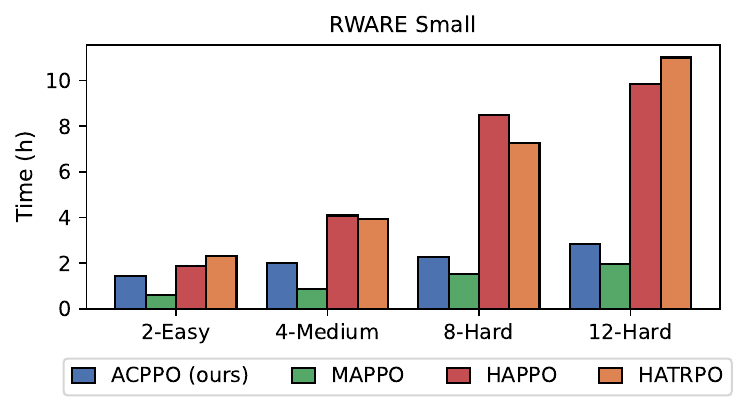}
    \caption{Wall-Clock Training Time on RWARE (5M timesteps)}
    \label{fig:runtime}
\end{figure}
In Figure~\ref{fig:runtime}, we show the wall-clock training time of running MAPPO, HAPPO, HATRPO and ACPO on RWARE for 5M timesteps. With the number of agents increasing from 2 to 12, the runtime of ACPO remains comparable to MAPPO, and is substantially faster than alternating policy optimization methods such as HAPPO and HATRPO, which require alternately updating one agent at a time. Overall, ACPO significantly outperforms baselines in terms of return with only minimal additional computational overhead compared to MAPPO.

%%%%%%%%%%%%%%%%%%%%%%%%%%%%%%%%%%%%%%%%%%%%%%%%%%%%%%%%%%%%%%%%%%%
% Hyperparameter details
%%%%%%%%%%%%%%%%%%%%%%%%%%%%%%%%%%%%%%%%%%%%%%%%%%%%%%%%%%%%%%%%%%%
\section{Hyperparameter Details}
\label{appendix:hyperparameter_details}
For a fair comparison, we set the network type (MLP or GRU) and hidden layer size to be consistent across all algorithms. The total number of parameters used by ACPPO are set to be similar to that of other algorithms. The design choices follow the experimental setups of \citet{zhong2024haml} and \citet{papoudakis2021epymarl}.
The discount factor $\gamma$ is fixed, as it is inherent to the MMDP rather than a tunable hyperparameter. In contrast ACPO employs serialization, where the advantage is computed using $\gamma' = \gamma^{1/N}$.

\begin{table}[h!]
\vspace{5pt}
\caption{Common Parameters for Baseline Algorithms}
\label{table:hyperparameters}
\centering
\scriptsize{
\begin{tabular}{c|c|c|c|c}
\toprule
Parameter & RWARE & MA-MuJoCo (on-policy) &  MA-MuJoCo (off-policy) & SMACv2 \\
\midrule
Network  & MLP & MLP &  MLP & GRU \\
Hidden Sizes (48)  & [128, 256] &  [256, 256] & [128, 128] & [64] \\
$\gamma$ & 0.99 & 0.99 & 0.99 & 0.99 \\
\bottomrule
\end{tabular}
}
\vskip -0.1in
\end{table}

\begin{table}[h!]
\vspace{5pt}
\caption{Common Parameters for ACPO}
\label{table:hyperparameters_acpo}
\centering
\scriptsize{
\begin{tabular}{c|c|c|c|c}
\toprule
Parameter & RWARE & MA-MuJoCo (on-policy) &  MA-MuJoCo (off-policy) & SMACv2 \\
\midrule
Network  & MLP & MLP &  MLP & GRU \\
Hidden Sizes  & [128, 128] &  [256, 256] & [128, 128] & [48] \\
Hidden Sizes (Action Belief Network)  & [64, 64] &  $-$ & $-$ & [48, 48] \\
$\gamma$ & 0.99 & 0.99 & 0.99 & 0.99 \\
\bottomrule
\end{tabular}
}
\vskip -0.1in
\end{table}

For all baselines, we use the reported hyperparameters from 
\citet{papoudakis2021epymarl} for RWARE, \citet{ellis2023smacv2} for SMACv2 and tune additional hyperparameters for HAPPO and HATRPO for MA-MuJoCo~(Gymnasium) on top of the default hyperparameters from \citet{zhong2024haml} for MA-MuJoCo~(Gym). The full set of hyperparameters are provided in our code at the following link: \url{https://github.com/dematsunaga/agent-chained-policy-optimization}.

\section{Computational Resources}
\label{appendix:computational_resources} 

% small-2-easy / small-4-medium / small-8-hard / small-12-hard
For RWARE experiments, we utilized a single NVIDIA GeForce RTX 3090 graphics processing unit (GPU). 
The training times varied across algorithms and the number of agents.
For the 2,4,8, and 12-agent environments, ACPO took 14H, 19H, 20H and 23H, respectively. 
The corresponding times were MAPPO (4H, 6H, 13H, 16H), HAPPO (16H, 32H, 66H, 82H) and HATRPO (18H, 30H, 59H, 90H).

\end{document}